\documentclass[onecolumn]{fairmeta}

\usepackage{amsmath}
\usepackage{amssymb}
\usepackage{amsfonts}
\usepackage{amsthm}
\usepackage{bbm}
\usepackage{wrapfig}
\usepackage{algorithm}
\usepackage{algpseudocode}
\usepackage{url}

\usepackage{amsmath,amsfonts,bm}

\def\eqref#1{equation~\ref{#1}}
\def\1{\bm{1}}

\DeclareMathAlphabet{\mathsfit}{\encodingdefault}{\sfdefault}{m}{sl}
\SetMathAlphabet{\mathsfit}{bold}{\encodingdefault}{\sfdefault}{bx}{n}

\theoremstyle{plain}
\newtheorem{theorem}{Theorem}

\newtheorem{proposition}[theorem]{Proposition}
\newtheorem{corollary}[theorem]{Corollary}
\theoremstyle{definition}

\usepackage{caption}

\title{MOSS: Mode-Oriented Source Shaping for Robot Flow Matching Policies}

\author[1,2]{Zirun Ma\textsuperscript{*}}
\author[1]{Jindou Jia\textsuperscript{*}}
\author[1]{Tuo An}
\author[1]{Jiaqi Bai}
\author[1]{Jiahua Dong}
\author[1]{Jianfei Yang\textsuperscript{$\dagger$}}

\affiliation[1]{MARS Lab, Nanyang Technological University}
\affiliation[2]{Tsinghua University
    \newline
    \textsuperscript{*} Equal contribution,
    \textsuperscript{$\dagger$} Corresponding author.
}

\correspondence{Jianfei Yang (\email{jianfei.yang@ntu.edu.sg}).}
\metadata[Project Page]{\url{https://mmmmazr.github.io/moss}}

\abstract{
In robotic manipulation, the same observation often admits multiple feasible actions. Generative policies such as diffusion and flow matching naturally capture such multimodal behaviors, yet which mode is executed is left to random sampling and cannot be actively controlled. We propose \textbf{M}ode-\textbf{O}riented \textbf{S}ource \textbf{S}haping (\textbf{MOSS}), which turns mode selection into an explicit choice of where generation starts. We first reveal, both empirically and theoretically, that deterministic flow transport induces mode-dependent structure in the source space. Under ideal deterministic ODE transport, distinct behavioral modes have distinct source-space preimages, while practical approximation errors can blur their boundaries due to imbalanced data, limited model capacity, and few inference steps. Motivated by this source-space structure, MOSS constructs semantically aligned mode-specific \textit{Gaussian} sources that encode mode identity in the initial noise. A modality-strength estimator predicts the local degree of multimodality, which adaptively controls the geometry of the mode-specific sources. It contracts them toward a shared \textit{Gaussian} where behaviors are similar and separates them near behavioral branching points. During deployment, the desired behavior is selected at inference simply by sampling from the corresponding component. Across $6$ simulated and $3$ real-world manipulation tasks, MOSS steers to the target mode up to $100\%$ of the time with high task success, recovers minority modes at over $90\%$ guidance rate under a $10\!:\!90$ imbalance dataset, and surpasses mode-conditioned flow policy by $43.3$ points of guidance rate on real robots. 
}

\begin{document}
\maketitle

\section{Introduction} 

Generative models have recently become a prominent paradigm in robot imitation learning~\citep{pearce2023imitating, pi07}. Unlike conventional policies that regress toward an optimal action, representative approaches such as Diffusion Policy (DP)~\citep{Chi-RSS-23} and Flow Matching (FM)~\citep{lipman2023flow,bai2026flash} generate action trajectories from random noise, allowing them to represent multiple feasible behaviors under the same observation, i.e., \textit{multimodal} behaviors~\citep{zhai2025vfp}. Here, a \textit{behavior mode} denotes a semantically coherent family of action sequences that follow the same qualitative strategy for accomplishing the task~\citep{Chi-RSS-23}. For instance, a robot may approach the same object from the left or the right, or interact with different targets in different orders, while still completing the task successfully. However, in existing generative policies, which mode is ultimately executed is largely determined implicitly by random sampling~\citep{lipman2023flow, zhang2026slfm}. This limits their applicability in settings that require user-specified preferences, test-time constraints, or human intervention~\citep{zoellner2026temporal, he2025uncertainty}.

Existing solutions fall into two categories. The straightforward one explicitly conditions the generative model on a mode indicator, in practice most often a language instruction, so that different conditions induce different action distributions~\citep{instructvla, zhang2026affordance, rdt1b}. Such conditioning, however, often fails to steer the policy reliably when the training data are imbalanced~\citep{fang2026visionoverrideslanguageevaluating}, since high-dimensional visual scenes are strongly correlated with task execution, the model tends to learn scene-action shortcuts that bypass the weaker, low-dimensional language condition~\citep{ren2026disc}. The other line instead operates on the source distribution of FM, controlling multimodal behaviors by implicitly learning a set of \textit{Gaussian} priors~\citep{zhang2026slfm, machado2026lafm}. These priors, however, carry no explicit semantic correspondence to behavioral modes, so aligning them with user-defined modes still demands additional networks and laborious post-alignment~\citep{zhang2026slfm}.

Toward a more reliable handle on mode selection, we first analyze how a well-trained FM policy behaves on multimodal tasks. We find that, even with a standard isotropic \textit{Gaussian} source, the final behavior does not arise randomly from the entire noise space. Under deterministic Ordinary Differential Equation (ODE) transport, different behavioral modes tend to occupy distinct, albeit noticeably entangled, regions of the source space. We further show theoretically that, under ideal transport conditions, such mode-specific initial-noise regions are guaranteed to be implicitly separated. In practice, however, imbalanced mode distributions, limited model capacity, and a restricted number of inference steps leave these partitions interleaved, weakly separated, or severely imbalanced, sometimes resulting in mode preference or even mode collapse~\citep{zhai2025vfp,mazza2026understanding}. This raises a natural question: \textit{can we make the existing partitioning more explicit, and steer behavior directly through the initial distribution?}

Based on this insight, we propose \textbf{M}ode-\textbf{O}riented \textbf{S}ource \textbf{S}haping, termed \textbf{MOSS}, which enables controllable multimodal generation by dynamically shaping the source distribution of FM, as shown in Figure~\ref{fig:teaser}. Specifically, we first construct semantically defined mode-specific \textit{Gaussian} sources, explicitly encoding mode identity into the initial noise and making the resulting control more robust to imbalanced mode distributions, limited model capacity, and limited inference steps. Since multiple modes typically arise only during certain critical phases rather than throughout the entire task~\citep{jia2026marspolicy}, a modality-strength estimator then predicts the local degree of multimodality from previous observation and action history, and uses it to continuously modulate the separability of the \textit{Gaussian} components. When demonstrations from different modes behave alike, the components contract toward a shared \textit{Gaussian} and maximize data sharing. As the robot approaches a behavioral branch, they gradually separate and give the initial noise a clear, stable mode identity. 

% \begin{figure}
%     \centering
%     \vspace{-10pt}
%     \includegraphics[width=1.0\linewidth]{figure/teaser.png}
%     \vspace{-6pt}
%     \caption{\textbf{Overview of MOSS.} \textbf{(a)} A real-world \textit{open-drawer} task whose shared approach trajectory branches into two feasible modes. \textbf{(b)} Evolution of the initial noise. Standard FM starts from one isotropic \textit{Gaussian} and separates modes only implicitly. MOSS shapes the source by the estimated modality strength, merging the components where the modes behave alike and separating them at branching. \textbf{(c)} Architecture of MOSS. The modality-strength estimator sets the source separation, while a single flow network is shared across modes.}
%     \vspace{-15pt}
%     \label{fig:teaser}
% \end{figure}

\begin{figure*}[t]
    \centering
    \includegraphics[width=\textwidth]{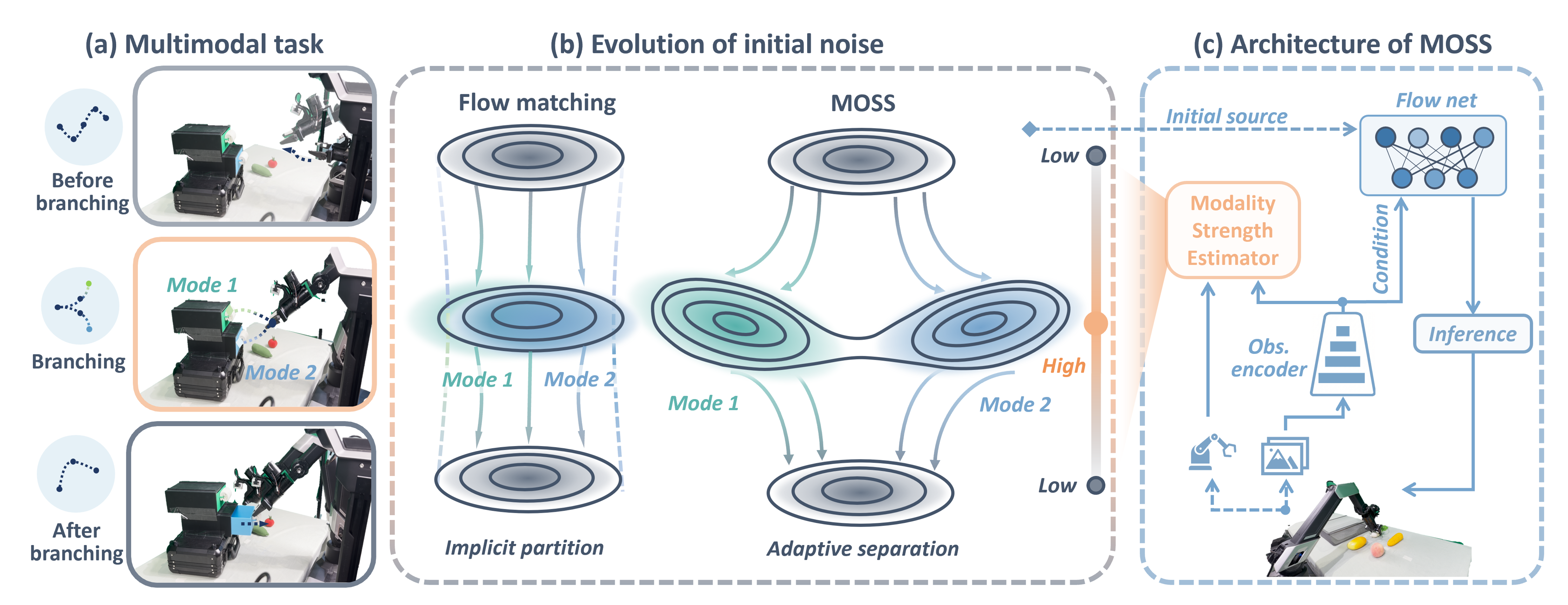}
    \caption{\textbf{Overview of MOSS.} \textbf{(a)} A real-world \textit{open-drawer} task whose shared approach trajectory branches into two feasible modes. \textbf{(b)} Evolution of the initial noise. Standard FM starts from one isotropic \textit{Gaussian} and separates modes only implicitly. MOSS shapes the source by the estimated modality strength, merging the components where the modes behave alike and separating them at branching. \textbf{(c)} Architecture of MOSS. The modality-strength estimator sets the source separation, while a single flow network is shared across modes.}
    \label{fig:teaser}
\end{figure*}

MOSS provides several practical advantages beyond explicit mode selection. First, when different modes in the training demonstrations are highly imbalanced, the generation probabilities of standard FM tend to inherit the empirical data distribution. In contrast, our method can directly adjust the deployment probability of each mode by changing the source sampling distribution. Second, for data-scarce modes, we find that the proposed strategy can effectively leverage shared trajectory data from other modes, allowing the policy to learn more stable common motion structures while preserving clear mode boundaries, thereby improving both sample efficiency and task success for minority modes. Finally, compared with methods that apply strong guidance to shift the learned vector field at inference time, mode-specific sources still initialize transport from regions already established during training. This can potentially reduce trajectory-quality degradation caused by inference-time guidance~\citep{yu2026conflict}, while naturally extending to currently emerging one-step denoising algorithms that do not support guidance~\citep{ho2022cfg, jia2026a2a}.

We systematically evaluate our method on multiple robotic manipulation tasks with multimodal structure, including $6$ simulation tasks and $3$ real-world robot tasks. Compared with standard FM~\citep{lipman2023flow}, explicit mode-conditioned FM~\citep{lei2026primaryfine}, and source-based methods~\citep{zhang2026slfm, luo2026mmfm}, MOSS achieves up to $100\%$ target-mode steering while maintaining high task success rate (SR). Even under a highly imbalanced mode distribution ($10:90$), our method can reliably recover the minority mode, achieving over $90\%$ guidance rate (GR). Ablations show that adaptive source geometry improves GR over fixed \textit{Gaussian} separation by $12$-$14$ percentage points in average GR at the same training budget. On real-world tasks, MOSS achieves an average GR/SR of $84.2\%/85.3\%$, outperforming mode-conditioned FM by $43.3$ and $12.0$ percentage points in average GR and SR, respectively. These results validate the effectiveness of source-space control under visual noise, limited demonstrations, and execution errors.

\section{Related work}

\paragraph{Multimodal generative policies.}
Generative policies such as diffusion~\citep{Chi-RSS-23} and FM~\citep{lipman2023flow} naturally model multimodal action distributions through stochastic generation. However, different behavioral modes can still become entangled or collapse during generation~\citep{zhai2025vfp,mazza2026understanding}. Recent methods improve multimodal modeling through additional structure. VFP~\citep{zhai2025vfp} combines a variational latent prior, optimal transport, and a MoE decoder to better separate action modes, while MARS~\citep{jia2026marspolicy} adaptively switches between stochastic and deterministic generation according to multimodality degree. None of these methods can actively control multimodality.

\paragraph{Source shaping for flow matching.}
Standard FM-based policies~\citep{lipman2023flow} commonly initialize transport from a simple \textit{Gaussian} source. Recent works instead design the source distribution to improve efficiency or structural alignment. A2A~\citep{jia2026a2a} initializes flow from action-history representations to enable shorter transport and one-step inference; VITA~\citep{gao2025vita} uses visual latents as the flow source; and LAFM~\citep{machado2026lafm} selects learned source priors according to latent actions. These methods primarily target efficient or better-aligned transport rather than explicit mode-level control. More recently, SL-FM~\citep{zhang2026slfm} separates behavioral branches through orthogonal source lifting while retaining a shared velocity field. MOSS similarly controls behavior through the source space, but explicitly aligns source components with semantic behavior modes and adaptively adjusts their separation according to local multimodality.

\paragraph{Explicit conditioning for behavioral control.}
Explicit conditioning provides a direct mechanism for controlling multimodal action generation by injecting the desired behavioral intent into the policy. Large VLA models such as RDT and $\pi_{0.5}$~\citep{rdt1b,pi07,pi05} use language and visual-semantic conditions to guide action generation. Lightweight generative policies can instead introduce explicit discrete mode variables; for example, PF-DAG~\citep{lei2026primaryfine} predicts a primary mode and generates continuous actions with a mode-conditioned MeanFlow policy. While effective, such explicit conditioning may introduce shortcut solutions, where the policy relies heavily on the injected mode signal rather than organizing multimodal behaviors intrinsically. In contrast, MOSS controls behavior through the source distribution without exposing the mode variable to the policy itself, and our experiments show source-based mode control yields stronger controllability than directly injecting the mode variable into the velocity-field condition.

\section{Multimodal Property of Flow Matching}
\label{sec:analysis}

Standard FM learns a continuous transport from a simple source to a complex action distribution conditioned on an observation $\mathbf{o}\in\mathcal{O}$~\citep{lipman2023flow,jia2026a2a}. Let $d\in\mathbb{N}$ denote the dimension of the vectorized action chunk, $\mathbf{z}\in\mathbb{R}^d$ with $\mathbf{z}\sim\mu_0=\mathcal{N}(\mathbf{0}_d,\mathbf{I}_d)$ be the source noise, where $\mathbf{0}_d\in\mathbb{R}^d$ and $\mathbf{I}_d\in\mathbb{R}^{d\times d}$, and let $\mathbf{y}\in\mathbb{R}^d$ with $\mathbf{y}\sim p^\mathbf{o}$ denote the target action distribution. Consider the linear \textit{Gaussian} probability path
\begin{equation}
\mathbf{x}_t = t\mathbf{y}+\sigma_t \mathbf{z},
\qquad
\sigma_t = 1-(1-\sigma_{\min})t,
\qquad
\sigma_{\min}\ge 0,
\label{eq:fm_probability_path}
\end{equation}
where $\mathbf{x}_t\in\mathbb{R}^d$ and $t\in[0,1]$. At $t=1$, the probability-path endpoint is $\mathbf{x}_1=\mathbf{y}+\sigma_{\min}\mathbf{z}$. We write $p_1^\mathbf{o}$ for its distribution and $\Phi_\mathbf{o}:\mathbb{R}^d\to\mathbb{R}^d$ for the flow map induced by the ideal FM ODE from $\mathbf{x}_0=\mathbf{z}$ to $\mathbf{x}_1$.

Empirically, we observe that FM policies trained on multimodal tasks already exhibit a structured source space: initial noises leading to different behavioral modes tend to occupy distinct regions of the \textit{Gaussian} source, albeit noticeably entangled near their interface, as shown in Figure~\ref{fig:fm_noise_separation}. We next show that this phenomenon arises naturally from deterministic transport theoretically.

\subsection{Noise Zoning and Its Geometric Regularity}

Partition the action space $\mathbb{R}^d$ into pairwise disjoint open regions $\mathcal{A}_1^\mathbf{o},\dots,\mathcal{A}_K^\mathbf{o}\subset\mathbb{R}^d$ with semantic boundary $\mathcal{B}_A^\mathbf{o}\subset\mathbb{R}^d$, where $K\in\mathbb{N}$ denotes the number of behavioral modes and
$p_1^\mathbf{o}(\mathcal{B}_A^\mathbf{o})=0$. Let $h_\mathbf{o}:\mathbb{R}^d\setminus \mathcal{B}_A^\mathbf{o}\to\{1,\dots,K\}$ denote the mode assignment of an endpoint $\mathbf{x}\in \mathcal{A}_k^\mathbf{o}$, i.e., $h_\mathbf{o}(\mathbf{x})=k$ for $\mathbf{x}\in \mathcal{A}_k^\mathbf{o}$. The \emph{noise zone} of mode $k$ and the source-space boundary are the preimages
$\mathcal{Z}_k^\mathbf{o}=\Phi_\mathbf{o}^{-1}(\mathcal{A}_k^\mathbf{o})\subset\mathbb{R}^d$ and $\mathcal{B}_Z^\mathbf{o}=\Phi_\mathbf{o}^{-1}(\mathcal{B}_A^\mathbf{o})\subset\mathbb{R}^d$.

\begin{theorem}[Noise zoning]
\label{thm:noise_zoning}
If $\Phi_\mathbf{o}$ is continuous and satisfies the exact transport condition $(\Phi_\mathbf{o})_\#\mu_0=p_1^\mathbf{o}$, then $\mathcal{Z}_1^\mathbf{o},\dots,\mathcal{Z}_K^\mathbf{o}$ are pairwise disjoint open sets with
\begin{equation}
\mu_0(\mathcal{Z}_k^\mathbf{o})=p_1^\mathbf{o}(\mathcal{A}_k^\mathbf{o}),
\qquad
\mu_0(\mathcal{B}_Z^\mathbf{o})=0 .
\label{eq:noise_zoning_mass}
\end{equation}
\end{theorem}
\begin{proof}
See Appendix~\ref{app:noise_zoning}
\end{proof}

Theorem~\ref{thm:noise_zoning} shows that, up to a null boundary, the \textit{Gaussian} source is already partitioned into semantic \emph{noise zones}, each carrying exactly the probability mass of its behavioral mode. This partition is not imposed by design but arises from deterministic transport itself, which is the basic rationale for steering multimodal behavior through the source space: if modes already occupy distinct source regions, the generated behavior can in principle be selected by choosing where transport is initialized, without conditioning or modifying the velocity field.

\begin{wrapfigure}{r}{0.32\textwidth}
    \centering
    \vspace{-10pt}
    \includegraphics[width=\linewidth]{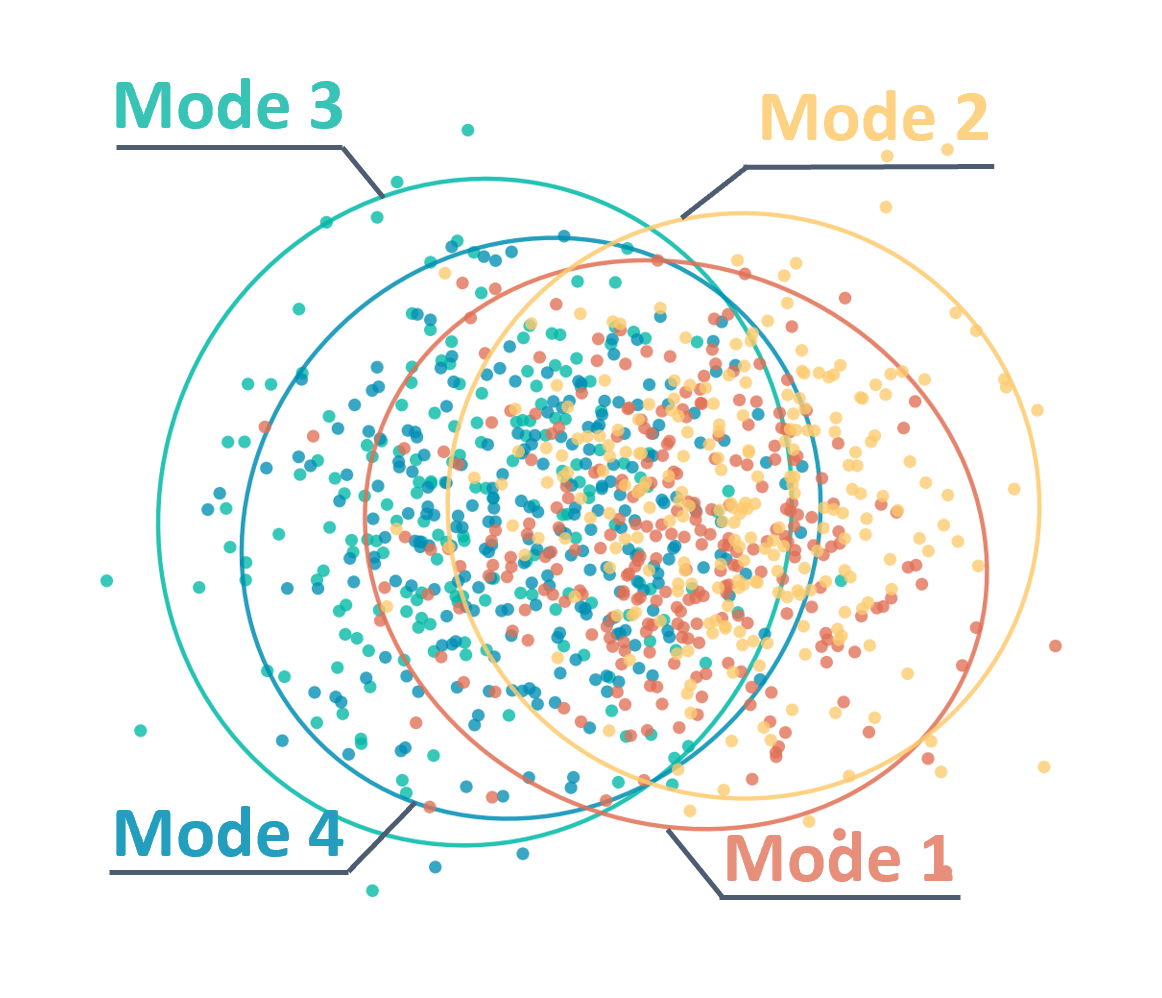}
    \vspace{-20pt}
    \caption{\textbf{Initial-noise mode separation in FM.}
    Experimental details are provided in Appendix~\ref{app:fm_noise_separation}.}
    \label{fig:fm_noise_separation}
    \vspace{-15pt}
\end{wrapfigure}

We now show that these zones are not arbitrarily fragmented sets but possess stable local geometry. For a concise global statement we take $\sigma_{\min}>0$. The case $\sigma_{\min}=0$ admits an analogous local guarantee as shown in Appendix~\ref{app:sigma_zero}. Under standard regularity conditions on the FM velocity field $\mathbf{v}_t^\mathbf{o}:\mathbb{R}^d\to\mathbb{R}^d$, Appendix~\ref{app:margin_separation} shows that the induced ODE flow is a bi-\textit{Lipschitz} homeomorphism. Specifically, there exists a finite constant $\Lambda_\mathbf{o}\in\mathbb{R}_{\ge 0}$ such that
\begin{equation}
\begin{aligned}
e^{-\Lambda_\mathbf{o}}\|\mathbf{z}-\mathbf{z}'\|_2
&\le
\|\Phi_\mathbf{o}(\mathbf{z})-\Phi_\mathbf{o}(\mathbf{z}')\|_2 \\
&\le
e^{\Lambda_\mathbf{o}}\|\mathbf{z}-\mathbf{z}'\|_2,
\end{aligned}
\quad
\forall \mathbf{z},\mathbf{z}'\in\mathbb{R}^d .
\label{eq:fm_bilipschitz}
\end{equation}

\begin{theorem}[Margin-induced source separation]
\label{thm:regularity}
Suppose $\Phi_\mathbf{o}$ is a bi-Lipschitz homeomorphism satisfying \eqref{eq:fm_bilipschitz}. Then $\mathcal{Z}_k^\mathbf{o}$ is open and connected whenever $\mathcal{A}_k^\mathbf{o}$ is.
Define the endpoint semantic margin
$\gamma_\mathbf{o}(\mathbf{z}):=\operatorname{dist}\left(\Phi_\mathbf{o}(\mathbf{z}),\mathcal{B}_A^\mathbf{o}\right)\in\mathbb{R}_{\ge 0}$
and the corresponding source-space radius
$R_\mathbf{o}(\mathbf{z}):=e^{-\Lambda_\mathbf{o}}\gamma_\mathbf{o}(\mathbf{z})\in\mathbb{R}_{\ge 0}$,
where $\operatorname{dist}(\mathbf{x},S):=\inf_{\mathbf{s}\in S}\|\mathbf{x}-\mathbf{s}\|_2$ for $\mathbf{x}\in\mathbb{R}^d$ and $S\subset\mathbb{R}^d$.
Then
\begin{equation}
\begin{aligned}
B\!\left(\mathbf{z},R_\mathbf{o}(\mathbf{z})\right)
&\subseteq \mathcal{Z}_k^\mathbf{o}, \\
\|\mathbf{z}-\mathbf{z}'\|_2
&\ge R_\mathbf{o}(\mathbf{z})+R_\mathbf{o}(\mathbf{z}'),
\end{aligned}
\label{eq:cross_mode_separation}
\end{equation}
for all $\mathbf{z}\in \mathcal{Z}_k^\mathbf{o}$ and
$\mathbf{z}'\in \mathcal{Z}_j^\mathbf{o}$ with $k\neq j$.
Here, $B(\mathbf{z},r):=\{\mathbf{u}\in\mathbb{R}^d:\|\mathbf{u}-\mathbf{z}\|_2<r\}$ denotes the open Euclidean ball centered at $\mathbf{z}\in\mathbb{R}^d$ with radius $r\ge 0$.
\end{theorem}

\begin{proof}
See Appendix~\ref{app:margin_separation}.
\end{proof}

Theorem~\ref{thm:regularity} shows that ideal FM induces not only distinct semantic noise zones, but also locally stable ones: behaviors lying farther from the semantic boundary correspond to a larger same-mode neighborhood in the source space.

\subsection{Practical FM Blurs the Noise Zoning}

In practice, finite demonstrations, imperfect optimization, and finite-step ODE integration make the realized endpoint map $\widetilde{\Phi}_\mathbf{o}:\mathbb{R}^d\to\mathbb{R}^d$ deviate from $\Phi_\mathbf{o}$. We measure this by the endpoint transport error 
$\eta_\mathbf{o}(\mathbf{z})=\|\widetilde{\Phi}_\mathbf{o}(\mathbf{z})-\Phi_\mathbf{o}(\mathbf{z})\|_2$

\begin{proposition}
\label{prop:mode_preservation}
For any $\mathbf{z}\in \mathcal{Z}_k^\mathbf{o}$, if $\eta_\mathbf{o}(\mathbf{z})<\gamma_\mathbf{o}(\mathbf{z})$ then
$h_\mathbf{o}(\widetilde{\Phi}_\mathbf{o}(\mathbf{z}))=h_\mathbf{o}(\Phi_\mathbf{o}(\mathbf{z}))$.
\end{proposition}

\textit{Proof.} See Appendix~\ref{app:practical_error}.

\begin{corollary}[Boundary-sensitive region]
\label{cor:boundary_region}
The guarantee vanishes once the transport error becomes comparable to the semantic margin: source points with $\gamma_\mathbf{o}(\mathbf{z})=O(\eta_\mathbf{o}(\mathbf{z}))$ form a boundary-sensitive region in which mode confusion is likely to occur.
\end{corollary}

Every practical factor pushes in the same direction: finite demonstrations make velocity estimation unreliable in sparsely covered regions, mode imbalance further removes supervision from minority behaviors, a network of bounded \textit{Lipschitz} constant cannot reproduce the sharp velocity variation that a clean mode separation demands~\citep{mazza2026understanding}, and finite-step integration adds numerical error. The resulting growth of $\eta_\mathbf{o}$ widens the ambiguous region and manifests during sampling as \emph{mode intersection}, \emph{mode preference}, or even local \emph{mode collapse}. This exposes a fundamental limitation of leaving noise zoning to be learned implicitly from a standard \textit{Gaussian} source: the existence of an ideal mode partition does not imply that sufficiently clear, stable, and controllable source regions emerge under finite data, limited model capacity, and finite-step inference.

\section{Mode-Oriented Source Shaping}

\subsection{Overview}
Section~\ref{sec:analysis} shows that FM naturally induces
mode-dependent regions in the source space, but this separation is implicit and can become weak or entangled in practice. Rather than recovering these implicit regions, MOSS deliberately imposes a simple, semantically aligned source geometry to make mode separation explicit and controllable. It enlarges the source-space separation between modes when it is actually needed, and selects the mode by choosing \emph{where the transport starts}. How much separation is needed depends on the state: when the upcoming behavior is highly multimodal, MOSS places the mode-specific sources far apart; when the behavior is essentially determined, it lets them contract back toward a shared \textit{Gaussian}. Figure~\ref{fig:teaser} depicts its architecture.

Concretely, given an observation history $\mathbf{o}$, a target action chunk $\mathbf{y}$, and a mode label $k$, MOSS predicts a modality strength $q\in[0,1]$ for the current state and draws the source sample $\mathbf{x}_0$ from the corresponding mode-specific component. In our implementation, we use $\sigma_{\min}=0$. Training then follows standard FM~\citep{lipman2023flow}, with a single velocity field shared by all modes:
\begin{equation}
\begin{aligned}
\mathbf{x}_t &= (1-t)\mathbf{x}_0 + t\mathbf{y}, \\
\mathcal{L}_{\mathrm{FM}}
&=
\mathbb{E}\left[
\left\|
\mathbf{v}_\theta(\mathbf{x}_t,t,\mathbf{o})-(\mathbf{y}-\mathbf{x}_0)
\right\|_2^2
\right].
\end{aligned}
\end{equation}

Mode selection is thereby moved out of the generative network and into the source space. The network learns one shared transport, while different behaviors are obtained by initializing that transport in different regions. In the following, we detail the architecture of MOSS.

\subsection{Modality Strength Estimator}
\label{sec:modality_strength_estimator}

The degree of behavioral multimodality can vary substantially along a robot trajectory~\citep{jia2026marspolicy}. We therefore dynamically estimate whether strong mode separation is necessary at the current state.

We first construct a multimodality supervision signal $r_i$ from the training data. For a query state $i$, we retrieve local neighbors with similar observation histories $\mathbf{o}_{\le t}$ and executed action histories $\mathbf{a}_{<t}$, and compare the between-mode variation of their future actions with the corresponding within-mode variation. Denoting these quantities by $\mathrm{Var}_{B,i}$ and $\mathrm{Var}_{W,i}$, respectively, we define
\begin{equation}
r_i
=
\frac{\mathrm{Var}_{B,i}}
{\mathrm{Var}_{B,i} + \beta \mathrm{Var}_{W,i} + \delta},
\end{equation}
where $\beta>0$ controls the relative contribution of within-mode variation, and $\delta>0$ is a small constant for numerical stability.

When the local neighborhood contains only a single plausible future mode, $\mathrm{Var}_{B,i}$ is close to zero and the resulting $r_i$ is small. In contrast, when multiple modes coexist and their future trajectories are clearly separated, $\mathrm{Var}_{B,i}$ increases and produces a larger modality strength. A large within-mode variation $\mathrm{Var}_{W,i}$ suppresses spurious multimodality signals caused by noisy or unstable neighbors.

We then train a lightweight MLP as the \emph{Modality Strength Estimator}:
\begin{equation}
w = s_\phi(\mathbf{o}_{\le t}, \mathbf{a}_{<t}),
\end{equation}
After training, $s_\phi$ is completely frozen and receives no gradient from the FM objective.

As visualized in Figure~\ref{fig:estimator}, the predicted strength peaks near states where multiple behavioral modes diverge and decreases once the mode has been resolved.
Since the numerical scale of $w$ can vary across tasks, we normalize it using task-wise percentile calibration computed solely from the training set, yielding a calibrated strength $q\in[0,1]$.
During rollout, we further apply temporal persistence to obtain $q^{\mathrm{eff}}$, preventing the separation strength from decaying before the selected behavior is completed.
For simplicity, we denote this effective strength as $q$ hereafter.
Additional estimator details are provided in Appendix~\ref{app:estimator}.

\subsection{Adaptive Mode-Oriented Source Shaping}

Given the estimated modality strength $q$, MOSS adaptively separates the source distributions associated with different behavioral modes. We present a unified formulation for $K$ modes.

To define the mode-discriminative subspace, we use the multimodality supervision $r_i$ from Section~\ref{sec:modality_strength_estimator} to identify training states where future behaviors diverge. Let $\bar{\mathbf{a}}_k$ denote the empirical mean future-action chunk of mode $k$ within this region. We then remove the component shared across modes,

\begin{equation}
\tilde{\mathbf{a}}_k
=
\bar{\mathbf{a}}_k
-
\frac{1}{K}\sum_{j=1}^{K}\bar{\mathbf{a}}_j,
\label{eq:action_mean}
\end{equation}
and let $\mathbf{U}\in\mathbb{R}^{d\times d_s}$ be an orthonormal basis of the subspace spanned by $\{\tilde{\mathbf{a}}_k\}_{k=1}^{K}$, where $d_s\le K-1$ denotes the dimension of the mode-discriminative subspace.

For a requested mode $k$, MOSS defines its source center as
\begin{equation}
\mathbf{c}_k(q)=\alpha(q)\tilde{\mathbf{a}}_k,
\end{equation}
and samples the initial condition from
\begin{equation}
p(\mathbf{x}_0\mid k,q)
=
\mathcal{N}\!\left(
\mathbf{c}_k(q),
\;
\sigma_{\parallel}^2(q)\mathbf{U}\mathbf{U}^\top
+
\sigma_{\perp}^2(\mathbf{I}-\mathbf{U}\mathbf{U}^\top)
\right).
\label{eq:moss_source}
\end{equation}

The modality strength $q$ controls both the displacement of the source centers and their overlap, where $\sigma_{\perp}$ is a fixed standard deviation in the orthogonal complement of the mode-discriminative subspace. Specifically, let
\begin{equation}
\Delta(q)
=
\min_{i\neq j}
\left\|
\mathbf{c}_i(q)-\mathbf{c}_j(q)
\right\|_2
\end{equation}
denote the minimum inter-mode center distance. We set
\begin{equation}
\sigma_{\parallel}(q)
=
\operatorname{clip}\left(
\frac{\Delta(q)}{2\rho(q)},
\sigma_{\parallel,\min},
\sigma_{\parallel,\max}
\right),
\label{eq:moss_sigma}
\end{equation}
where both $\alpha(q)$ and the target separation ratio $\rho(q)$ increase monotonically with $q$. Hence, highly multimodal states induce more widely separated and less overlapping mode-specific sources, whereas small $q$ keeps the source distributions closer to the standard Gaussian prior. For $K=2$, the centered mode means are opposite along a one-dimensional subspace, reducing ~\eqref{eq:moss_source} to the symmetric dual-Gaussian construction.

Importantly, the mode label is used only to select the initial source distribution and is never provided to the velocity field $\mathbf{v}_\theta$. MOSS therefore retains a single shared FM velocity field and realizes behavioral control solely through its initial condition. At inference time, a user or high-level planner only needs to specify the desired mode $k$ and sample $\mathbf{x}_0$ from the corresponding source distribution. Further implementation details of the source shaping procedure are provided in Appendix~\ref{app:source_shaping}.

\section{Evaluation}

\subsection{Experimental Setup}

We evaluate MOSS on $6$ simulated tasks and $3$ real-world robotic tasks, as shown in Figure~\ref{fig:experiments}. The simulation benchmark consists of $4$ bimodal tasks---\emph{Push Cube}, \emph{Grasp Eyeglasses}, \emph{Stack Cube}, and \emph{Put Wine}---and two multi-mode tasks, \emph{2D Navigation} and \emph{Place Cube}~\citep{jia2026marspolicy, mu2021maniskill, liu2023libero}. Each task admits multiple distinct yet valid strategies for successful completion. For real-world evaluation, we consider \emph{Place Cups}, \emph{Pick Veg}, and \emph{Open Drawer} on a Galaxea R1 Lite robotic platform.

\begin{figure*}[t]
    \centering
    \vspace{-30pt}
    \includegraphics[width=0.8\textwidth]{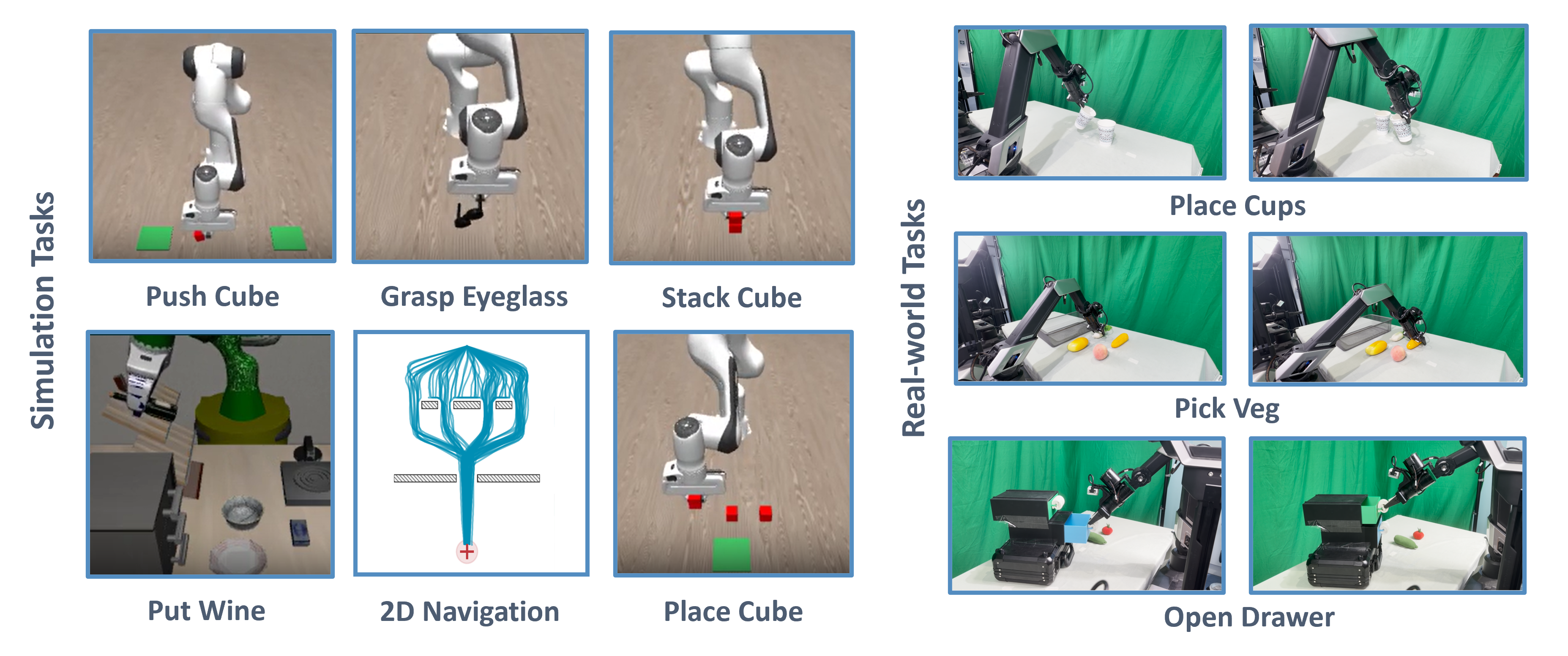}
    \vspace{-10pt}
    \captionsetup{justification=centering}
    \caption{\textbf{Overview of the simulation and real-world tasks used for evaluation.}}
    \vspace{-20pt}
    \label{fig:experiments}
\end{figure*}

% \begin{figure*}[t]
%     \centering
%     \vspace{-10pt}
%     \includegraphics[width=0.6\textwidth]{figures/estimator.png}
%     \captionsetup{justification=centering}
%     \caption{\textbf{Modality strength over task progress.} 
%     }
%     \label{fig:estimator}
%     \vspace{-10pt}
% \end{figure*}

\begin{wrapfigure}{r}{0.55\textwidth}
    \centering
    \vspace{-10pt}
    \includegraphics[width=\linewidth]{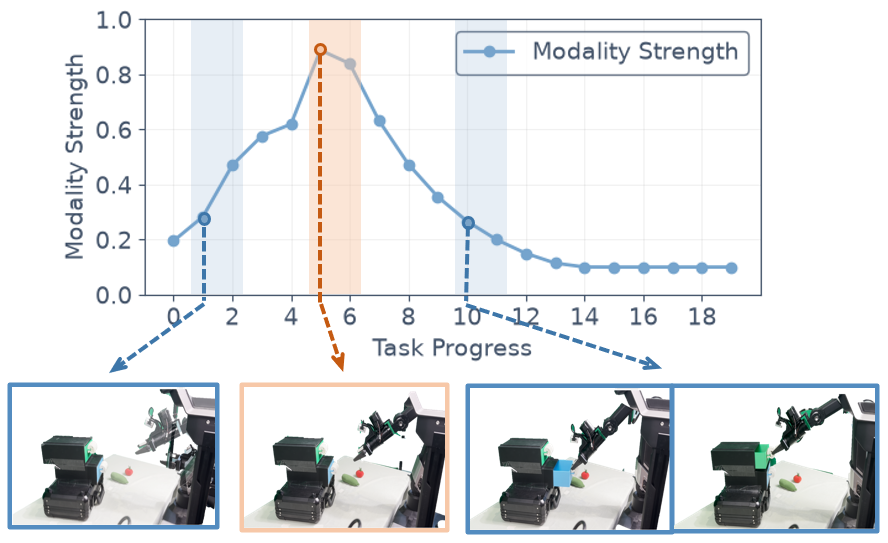}
    \vspace{-20pt}
    \captionsetup{justification=centering}
    \caption{\textbf{Modality strength over task progress.} 
    % The estimator peaks near the behavioral branching stage and decreases once the future motion becomes increasingly determined; larger modality strength triggers stronger source separation in MOSS.
    }
    \label{fig:estimator}
    \vspace{-10pt}
\end{wrapfigure}

We compare MOSS against $8$ baselines. Unconditional generative policies include FM-DiT~\citep{lipman2023flow,peebles2023scalable}, DDPM-DiT~\citep{ho2020denoising,peebles2023scalable}, ACT~\citep{zhao2023learning}, BET~\citep{shafiullah2022behavior}, A2A~\citep{jia2026a2a}, and MARS Policy~\citep{jia2026marspolicy}. We further compare against \emph{Mode-conditioned FM}, following mode-conditioned flow-matching formulations~\citep{lei2026primaryfine},
which provides the desired behavioral mode as an additional conditioning
signal to the velocity field, and \emph{GMM-Sourced FM}, following
multimodal source-space formulations~\citep{luo2026mmfm,zhang2026slfm},
which uses data-driven \emph{Gaussian} mixture components without explicit semantic mode supervision.
Key hyperparameters are standardized across models to ensure a fair
comparison. Additional experimental settings and evaluation details are provided in
Appendix~\ref{app:experiment_setup}.

Our evaluation targets three questions: (i) can MOSS steer the policy to a specified mode as expected; (ii) can it recover a minority mode under strongly imbalanced training data; and (iii) does it also bring other benefits, like improving task success and training efficiency?

\subsection{Controllable Multimodal Behavior}

\begin{table*}[t]
\centering
\caption{\textbf{Evaluation results on $4$ simulation tasks.} For each requested mode, we report GR and SR (GR / SR). 
% MOSS achieves consistently high controllability while maintaining task success across different behavioral modes. 
Best GR for each requested mode is shown in bold.}
\label{tab:sim_bimodal}
\vspace{-10pt}
\resizebox{\textwidth}{!}{
\begin{tabular}{lcccccccc}
\toprule
& \multicolumn{2}{c}{Push Cube}
& \multicolumn{2}{c}{Grasp Eyeglasses}
& \multicolumn{2}{c}{Stack Cube}
& \multicolumn{2}{c}{Put Wine} \\
\cmidrule(lr){2-3}
\cmidrule(lr){4-5}
\cmidrule(lr){6-7}
\cmidrule(lr){8-9}
Method
& Mode 1 & Mode 2
& Mode 1 & Mode 2
& Mode 1 & Mode 2
& Mode 1 & Mode 2 \\
\midrule

Mode-Cond. FM
& 0.86 / 0.96 & 0.94 / 1.00
& \textbf{1.00} / 1.00 & \textbf{1.00} / 1.00
& 0.96 / 0.96 & 0.94 / 0.94
& 0.92 / 0.92 & \textbf{0.94} / 0.96
\\

GMM-Sourced FM
& 0.72 / 0.96 & 0.92 / 1.00
& 0.98 / 1.00 & 0.42 / 1.00
& 0.58 / 1.00 & 0.56 / 1.00
& \textbf{0.94} / 0.94 & 0.80 / 0.94
\\

\textbf{MOSS}
& \textbf{0.99} / 0.99 & \textbf{1.00} / 1.00
& \textbf{1.00} / 1.00 & \textbf{1.00} / 1.00
& \textbf{0.98} / 0.98 & \textbf{0.96} / 0.96
& 0.92 / 0.94 & 0.86 / 0.94
\\

\bottomrule
\end{tabular}
}
\centering
\vspace{-10pt}

\end{table*}

\textbf{Simulation results.} Table~\ref{tab:sim_bimodal} summarizes controllable behavior generation on the four bimodal simulation tasks. MOSS consistently achieves high SR and GR for both requested modes, attaining the strongest average mode controllability while maintaining high task success. Mode-conditioned FM is also competitive in simulation, whereas GMM-Sourced FM exhibits substantially more task-dependent semantic alignment. As shown in Figure~\ref{fig:simulator_unguided}, unconditional policies can represent multiple behavioral outcomes but provide no mechanism for selecting a desired mode, often resulting in task-dependent mode preference.

\begin{figure*}[t]
    \centering
    % \vspace{-20pt}
    \includegraphics[width=0.9\textwidth]{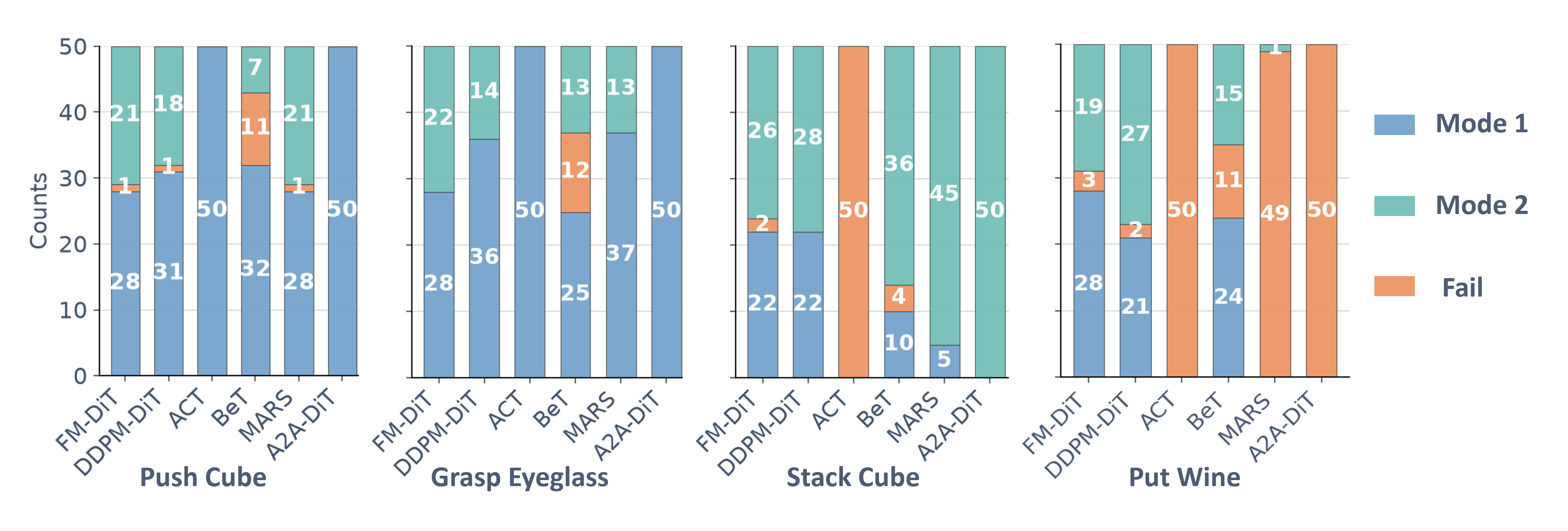}
    \vspace{-10pt}
    \captionsetup{justification=centering}
    \caption{\textbf{Behavior distributions of unconditional policies on simulation tasks.}}
    % \vspace{-20pt}
    \label{fig:simulator_unguided}
\end{figure*}

\begin{wrapfigure}{r}{0.50\textwidth}
    \centering
    \vspace{-20pt}
    \includegraphics[width=0.95\linewidth]{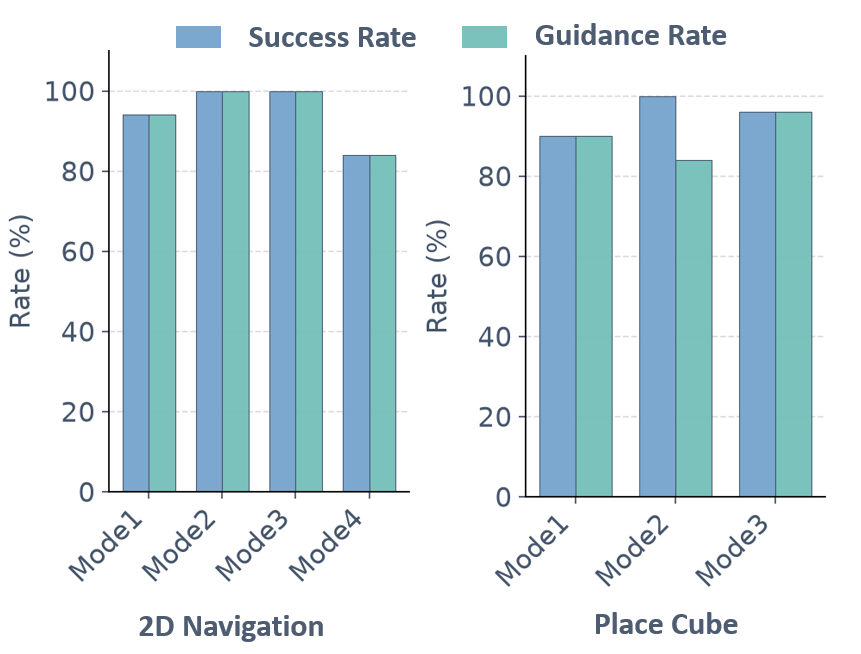}
    % \vspace{-20pt}
    \captionsetup{justification=centering}
    \caption{\textbf{Results on \emph{2D Navigation} and \emph{Place Cube}.}}
    \vspace{-20pt}
    \label{fig:simulator_multi}
\end{wrapfigure}

MOSS also generalizes beyond binary behaviors. As shown in Figure~\ref{fig:simulator_multi}, MOSS maintains high SR and GR across multiple requested modes on \emph{2D Navigation} and \emph{Place Cube}. Figure~\ref{fig:estimator} shows that the estimated modality strength remains low when different behaviors share similar motion and peaks near states where future trajectories diverge. Since the source separation in MOSS increases monotonically with $q$, these peaks trigger stronger mode-specific source separation around decision-critical stages.

\textbf{Real-world results.}
As shown in Table~\ref{tab:real_world_results}, MOSS achieves strong SR and GR on all three real-world tasks and is substantially more balanced across behavioral modes than Mode-conditioned FM. For example, Mode-conditioned FM fails to control mode 1 on \emph{Pick Veg} and mode 2 on \emph{Open Drawer}, whereas MOSS retains controllability for both modes. These results suggest that source-space control remains effective despite the visual and execution variability encountered in real-world manipulation. Increasing classifier-free guidance (CFG)~\citep{ho2022cfg} can partially improve Mode-conditioned FM, but does not consistently recover both SR and GR. Details are provided in Appendix~\ref{app:cfg_results}.

% \begin{table}[!htbp]
%     \centering
%     \caption{
%         \textbf{Real-world results.}
%         We report average SR and mode-specific
%         GR$_1$ and GR$_2$.
%         Best GR for each requested mode is shown in bold.
%     }
%     \label{tab:real_world_results}

%     \small
%     \begin{tabular}{lccc}
%         \toprule
%         Method & SR & GR$_1$ & GR$_2$ \\
%         \midrule
%         \multicolumn{4}{l}{\textit{Place Cups}} \\
%         FM-DiT & 0.70 & -- & -- \\
%         Mode-conditioned FM & 0.70 & 0.15 & 0.85 \\
%         \textbf{MOSS} & 0.98 & \textbf{0.95} & \textbf{1.00} \\
%         \midrule
%         \multicolumn{4}{l}{\textit{Pick Veg}} \\
%         FM-DiT & 0.60 & -- & -- \\
%         Mode-conditioned FM & 0.65 & 0.00 & 0.60 \\
%         \textbf{MOSS} & 0.75 & \textbf{0.80} & \textbf{0.65} \\
%         \midrule
%         \multicolumn{4}{l}{\textit{Open Drawer}} \\
%         FM-DiT & 0.75 & -- & -- \\
%         Mode-conditioned FM & 0.85 & 0.85 & 0.00 \\
%         \textbf{MOSS} & 0.83 & \textbf{0.90} & \textbf{0.75} \\
%         \bottomrule
%     \end{tabular}
% \end{table}

\subsection{Robustness and Training Efficiency}

To evaluate learning under biased demonstrations, we construct $10{:}90$ mode-imbalanced datasets on \emph{Push Cube}, \emph{Grasp Eyeglasses}, and \emph{Stack Cube}, and always evaluate the minority mode. MOSS and Mode-conditioned FM are trained on all $100$ demonstrations, whereas FM-DiT is trained only on the $10$ minority-mode demonstrations. Thus, all methods observe the same number of target-mode examples, while the former two can additionally exploit motion shared with the majority mode.

\begin{wraptable}{r}{0.45\textwidth}
    \centering
    \vspace{-20pt}
    \caption{\textbf{Real-world results.} We report average SR and mode-specific GR$_1$ and GR$_2$. Best GR for each requested mode is shown in bold.}
    \label{tab:real_world_results}
    % \vspace{-10pt}
    \small
    \begin{tabular}{lccc}
        \toprule
        Method & SR & GR$_1$ & GR$_2$ \\
        \midrule
        \multicolumn{4}{l}{\textit{Place Cups}} \\
        FM-DiT & 0.70 & -- & -- \\
        Mode-conditioned FM & 0.70 & 0.15 & 0.85 \\
        \textbf{MOSS} & 0.98 & \textbf{0.95} & \textbf{1.00} \\
        \midrule
        \multicolumn{4}{l}{\textit{Pick Veg}} \\
        FM-DiT & 0.60 & -- & -- \\
        Mode-conditioned FM & 0.65 & 0.00 & 0.60 \\
        \textbf{MOSS} & 0.75 & \textbf{0.80} & \textbf{0.65} \\
        \midrule
        \multicolumn{4}{l}{\textit{Open Drawer}} \\
        FM-DiT & 0.75 & -- & -- \\
        Mode-conditioned FM & 0.85 & 0.85 & 0.00 \\
        \textbf{MOSS} & 0.83 & \textbf{0.90} & \textbf{0.75} \\
        \bottomrule
    \end{tabular}
    \vspace{-20pt}
    
\end{wraptable}

\begin{figure*}[t]
    \centering
    \vspace{-10pt}
    \includegraphics[width=0.9\textwidth]{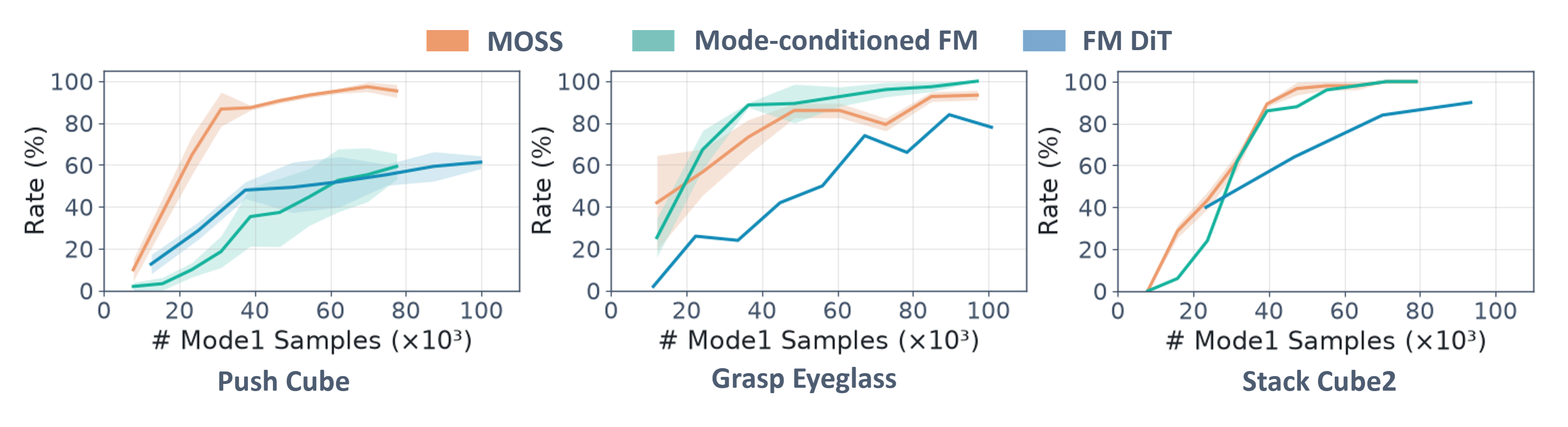}
    \vspace{-10pt}
    \caption{\textbf{Minority-mode learning efficiency under imbalanced datasets.} Further details are provided in Appendix~\ref{app:imbalance_results}.}
    \vspace{-5pt}
    \label{fig:imbalanced_efficiency}
\end{figure*}

As shown in Figure~\ref{fig:imbalanced_efficiency}, MOSS exhibits strong minority-mode sample efficiency, reaching high GR with relatively few target-mode sample exposures. The advantage is particularly pronounced on \emph{Push Cube}, where MOSS approaches its final performance substantially earlier than both baselines. Compared with FM-DiT trained only on minority demonstrations, these results suggest that the shared velocity field can reuse motion structure learned from the majority mode rather than learning the minority behavior entirely from scratch. Overall, MOSS provides robust minority-mode control while retaining high training efficiency under strongly imbalanced demonstrations. Implementation details are provided in Appendix~\ref{app:imbalance_results}.

\subsection{Ablation Studies}
\label{sec:ablation}

\begin{table*}[t]
\centering
\caption{\textbf{Ablation on $3$ simulation tasks.} We report mode-specific guidance rates (GR$_1$ and GR$_2$).  Best GR for each requested mode is shown in bold.}
\label{tab:ablation}
\vspace{-5pt}
\setlength{\tabcolsep}{5pt}
\renewcommand{\arraystretch}{1.05}
\resizebox{0.6\textwidth}{!}{%
\begin{tabular}{lccc}
\toprule
\textbf{Method}
& \textbf{Push Cube}
& \textbf{Grasp Eyeglasses}
& \textbf{Stack Cube} \\
\midrule
MOSS (single \emph{Gaussian})
& 0.50 / 0.68
& 0.76 / 0.80
& 0.58 / 0.70 \\
MOSS (single-step separation)
& 0.64 / 0.62
& 0.66 / 0.72
& 0.54 / 0.68 \\
MOSS (w/o adaptive)
& 0.86 / 0.85
& 0.86 / 0.90
& 0.78 / 0.88 \\
\textbf{MOSS (ours)}
& \textbf{0.99 / 1.00}
& \textbf{1.00 / 1.00}
& \textbf{0.98 / 0.96} \\
\bottomrule
\end{tabular}%
}

\vspace{-15pt}
\end{table*}

Table~\ref{tab:ablation} isolates three design choices in MOSS. Specifically, the variants test whether mode-specific Gaussian sources are necessary, whether separation at only the most multimodal action position is sufficient, and whether adaptive shaping improves over fixed separation. Using a single shared Gaussian yields poor controllability, while single-step separation provides limited improvement; extending separation to the full action chunk substantially increases GR. Adaptive source shaping further improves over fixed full-chunk separation and typically reaches strong performance earlier during training. Additional analyses are provided in Appendix~\ref{app:ablation}.

\section{Conclusion}

We presented MOSS, a source-space approach for controllable multimodal FM policies. Our analysis shows that ideal deterministic FM transport naturally induces mode-dependent noise regions, while practical approximation errors can blur these implicit partitions and lead to mode preference or collapse. MOSS makes this structure explicit by associating semantic behavioral modes with mode-specific source distributions and adaptively adjusting their separation according to local multimodality, without conditioning the shared velocity field on the requested mode. Across simulated and real-world manipulation tasks, MOSS achieves strong mode controllability and task success rate, extends naturally to multiple behavioral modes, and improves minority-mode learning under highly imbalanced demonstrations. Together, these results highlight source-space geometry as an effective design dimension for controllable generative robot policies.

\section{Limitations}

MOSS has two main limitations. First, the current formulation assumes a fixed set of discrete behavioral modes specified before training, and therefore does not directly support open-ended mode control or continuous behavioral variation. Extending MOSS toward dynamic mode discovery and continuous mode spaces remains an important direction for future work. Second, although the source-space formulation is in principle compatible with larger generative robot policies, its performance in such large-scale settings remains to be evaluated.

\newpage

\bibliography{paper}
\bibliographystyle{assets/plainnat}

\newpage
\beginappendix

\section{Visualization of Initial-Noise Mode Separation}
\label{app:fm_noise_separation}

We investigate whether the initial noise of a standard FM policy is predictive of the behavioral mode ultimately executed by the policy. We conduct the experiment on \emph{2D Navigation}, where a DiT-based FM policy can reach the goal through four feasible routes, denoted as Mode 1 -- Mode 4.

We perform 1,000 closed-loop rollouts using the standard Gaussian source distribution $\mathcal{N}(\mathbf{0},\mathbf{I})$. For each rollout, we record the initial noise of the first executed action chunk, which contains $8$ two-dimensional actions. Since route selection is primarily reflected in vertical motion, we retain only the $y$-dimension, resulting in an eight-dimensional noise feature for each rollout.

A rollout is considered successful if its final position is within $0.06$ of the goal at $(1.8,0.5)$. Successful trajectories are assigned mode labels according to their mean vertical coordinate in the wall-crossing region $x\in[0.78,0.85]$:
\[
\begin{aligned}
y < 0.22 &: \quad \text{Mode 1},\\
0.22 \leq y < 0.50 &: \quad \text{Mode 2},\\
0.50 \leq y < 0.78 &: \quad \text{Mode 3},\\
y \geq 0.78 &: \quad \text{Mode 4}.
\end{aligned}
\]
After excluding failed rollouts, 939 trajectories remain, with 188, 242, 267, and 242 samples for Mode 1 -- Mode 4, respectively.

For visualization, the eight-dimensional noise features are standardized and projected to two dimensions using LDA fitted with the trajectory-derived mode labels. LDA identifies directions that maximize between-mode relative to within-mode variation. Each point in Figure~\ref{fig:fm_noise_separation} represents the initial noise of one successful rollout, colored by its resulting behavioral mode, while the ellipses show the nominal $95\%$ Gaussian covariance contours of each mode.

We further quantify the semantic structure of this source-space partition. For each mode, we compute the mean of its initial-noise features. The mean $y$-noise values for Mode 1 -- Mode 4 are $-0.303$, $-0.182$, $0.135$, and $0.286$, respectively, increasing monotonically with the vertical location of the corresponding route. This suggests that the source-space organization learned by standard FM is aligned with the geometry of the resulting behaviors.

Together, these results show that a standard FM policy already develops a mode-relevant organization of its initial-noise space, despite being trained without explicit mode conditioning. However, the substantial overlap between neighboring modes also shows that this naturally emerging partition is insufficient for reliable mode control. This observation motivates MOSS to explicitly shape and separate mode-specific source regions for controllable multimodal generation.

To test whether this structure can directly influence generation, we construct a Gaussian source centered at each mode-specific mean while retaining the standard deviation of the original source, $\sigma=1$, and perform 100 rollouts per source. The resulting mode distributions are reported in Table~\ref{tab:fm_noise_mean_guidance}. Sampling around different mode-specific means induces distinct biases in the generated behaviors. In particular, the sources centered at the Mode 1 and Mode 4 means produce their corresponding modes in $62\%$ and $56\%$ of rollouts, respectively. The intermediate means yield more overlapping mode distributions, indicating that the naturally emerging source partition is semantically structured but not cleanly separated.

\section{Proofs for Section~\ref{sec:analysis}}
\label{app:theory}

% ============================================================
\subsection{Proof of Theorem~\ref{thm:noise_zoning}}
\label{app:noise_zoning}

\begin{proof}
Since $\mathcal{A}_k^{\mathbf{o}}\subset\mathbb{R}^d$ is open and
$\Phi_{\mathbf{o}}:\mathbb{R}^d\rightarrow\mathbb{R}^d$ is continuous,
its preimage
\begin{equation}
\mathcal{Z}_k^{\mathbf{o}}
=
\Phi_{\mathbf{o}}^{-1}
\left(
\mathcal{A}_k^{\mathbf{o}}
\right)
\label{eq:app_noise_zone_preimage}
\end{equation}
is open.

For any $i\neq j$, the semantic regions are pairwise disjoint. Hence
\begin{align}
\mathcal{Z}_i^{\mathbf{o}}
\cap
\mathcal{Z}_j^{\mathbf{o}}
&=
\Phi_{\mathbf{o}}^{-1}
\left(
\mathcal{A}_i^{\mathbf{o}}
\right)
\cap
\Phi_{\mathbf{o}}^{-1}
\left(
\mathcal{A}_j^{\mathbf{o}}
\right)
\nonumber\\
&=
\Phi_{\mathbf{o}}^{-1}
\left(
\mathcal{A}_i^{\mathbf{o}}
\cap
\mathcal{A}_j^{\mathbf{o}}
\right)
=
\varnothing,
\label{eq:app_noise_zone_disjoint}
\end{align}
so the corresponding noise zones are also pairwise disjoint.

\begin{table*}[t]
\vspace{-30pt}
\centering
\captionsetup{justification=centering}
\caption{\textbf{Results when sampling from Gaussian sources centered at the mode-specific initial-noise means of a standard FM policy.} Each row contains 100 rollouts with $\sigma=1$.}
\label{tab:fm_noise_mean_guidance}

\small
\resizebox{0.6\textwidth}{!}{%
\begin{tabular}{c c c c c c c}
\toprule
Target mean & Mode 1 & Mode 2 & Mode 3 & Mode 4 & Stuck & Target mode \\
\midrule
$-0.303$ (Mode 1) & \textbf{62} & 27 & 8  & 2  & 1 & 62\% \\
$-0.182$ (Mode 2) & 52 & \textbf{29} & 9  & 6  & 4 & 29\% \\
$ 0.135$ (Mode 3) & 6  & 15 & \textbf{34} & 43 & 2 & 34\% \\
$ 0.286$ (Mode 4) & 3  & 8  & 29 & \textbf{56} & 4 & 56\% \\
\bottomrule
\end{tabular}%
}

\end{table*}

By the exact transport condition
$(\Phi_{\mathbf{o}})_{\#}\mu_0=p_1^{\mathbf{o}}$,
the definition of pushforward measure gives, for every measurable
$S\subseteq\mathbb{R}^d$,
\begin{equation}
\mu_0
\left(
\Phi_{\mathbf{o}}^{-1}(S)
\right)
=
p_1^{\mathbf{o}}(S).
\label{eq:app_pushforward_identity}
\end{equation}
Taking $S=\mathcal{A}_k^{\mathbf{o}}$ yields
\begin{equation}
\mu_0
\left(
\mathcal{Z}_k^{\mathbf{o}}
\right)
=
p_1^{\mathbf{o}}
\left(
\mathcal{A}_k^{\mathbf{o}}
\right).
\label{eq:app_noise_zone_mass}
\end{equation}
Similarly, taking $S=\mathcal{B}_A^{\mathbf{o}}$ and using
$p_1^{\mathbf{o}}(\mathcal{B}_A^{\mathbf{o}})=0$ gives
\begin{equation}
\mu_0
\left(
\mathcal{B}_Z^{\mathbf{o}}
\right)
=
p_1^{\mathbf{o}}
\left(
\mathcal{B}_A^{\mathbf{o}}
\right)
=
0.
\label{eq:app_noise_boundary_mass}
\end{equation}
This proves Theorem~\ref{thm:noise_zoning}.
\end{proof}

% ============================================================
\subsection{Flow Regularity and Proof of Theorem~\ref{thm:regularity}}
\label{app:margin_separation}

We first justify the bi-Lipschitz property used in
Theorem~\ref{thm:regularity}. Let
$\Phi_{s,t}^{\mathbf{o}}:\mathbb{R}^d\rightarrow\mathbb{R}^d$
denote the FM ODE flow from time $s$ to $t$, with
$\Phi_{\mathbf{o}}=\Phi_{0,1}^{\mathbf{o}}$. Define
\begin{equation}
L_t^{\mathbf{o}}
:=
\sup_{\mathbf{x}\in\mathbb{R}^d}
\left\|
\nabla_{\mathbf{x}}
\mathbf{v}_t^{\mathbf{o}}(\mathbf{x})
\right\|_{\mathrm{op}},
\qquad
\Lambda_{\mathbf{o}}
:=
\int_0^1
L_t^{\mathbf{o}}
\,\mathrm{d}t
<
\infty,
\label{eq:app_global_lipschitz}
\end{equation}
and assume that the ODE is globally well posed on $[0,1]$.

For
$\mathbf{x}_t=\Phi_{0,t}^{\mathbf{o}}(\mathbf{z})$ and
$\mathbf{x}'_t=\Phi_{0,t}^{\mathbf{o}}(\mathbf{z}')$, we have
\begin{equation}
\left\|
\mathbf{v}_t^{\mathbf{o}}(\mathbf{x}_t)
-
\mathbf{v}_t^{\mathbf{o}}(\mathbf{x}'_t)
\right\|_2
\le
L_t^{\mathbf{o}}
\left\|
\mathbf{x}_t-\mathbf{x}'_t
\right\|_2.
\label{eq:app_velocity_lipschitz}
\end{equation}
Gr\"onwall's inequality therefore gives
\begin{equation}
\left\|
\Phi_{\mathbf{o}}(\mathbf{z})
-
\Phi_{\mathbf{o}}(\mathbf{z}')
\right\|_2
\le
e^{\Lambda_{\mathbf{o}}}
\left\|
\mathbf{z}-\mathbf{z}'
\right\|_2.
\label{eq:app_forward_lipschitz}
\end{equation}
Applying the same argument to the backward flow
$\Phi_{1,0}^{\mathbf{o}}=\Phi_{\mathbf{o}}^{-1}$ yields
\begin{equation}
e^{-\Lambda_{\mathbf{o}}}
\left\|
\mathbf{z}-\mathbf{z}'
\right\|_2
\le
\left\|
\Phi_{\mathbf{o}}(\mathbf{z})
-
\Phi_{\mathbf{o}}(\mathbf{z}')
\right\|_2
\le
e^{\Lambda_{\mathbf{o}}}
\left\|
\mathbf{z}-\mathbf{z}'
\right\|_2,
\label{eq:app_bilipschitz}
\end{equation}
which establishes the bi-Lipschitz property in
~\eqref{eq:fm_bilipschitz}.

\paragraph{Condition $\sigma_{\min}>0$.}
For the probability path in
~\eqref{eq:fm_probability_path}, define
\begin{align}
\mathbf{m}_t^{\mathbf{o}}(\mathbf{x})
&:=
\mathbb{E}
\left[
\mathbf{y}
\mid
\mathbf{x}_t=\mathbf{x},\mathbf{o}
\right],
\label{eq:app_conditional_mean}
\\
\boldsymbol{\Sigma}_t^{\mathbf{o}}(\mathbf{x})
&:=
\operatorname{Cov}
\left[
\mathbf{y}
\mid
\mathbf{x}_t=\mathbf{x},\mathbf{o}
\right].
\label{eq:app_conditional_covariance}
\end{align}
Since
$\dot{\sigma}_t=\sigma_{\min}-1$ and
$\sigma_t-t\dot{\sigma}_t=1$,
the population FM velocity satisfies
\begin{equation}
\mathbf{v}_t^{\mathbf{o}}(\mathbf{x})
=
\frac{\dot{\sigma}_t}{\sigma_t}\mathbf{x}
+
\frac{1}{\sigma_t}
\mathbf{m}_t^{\mathbf{o}}(\mathbf{x}),
\label{eq:app_population_velocity}
\end{equation}
with
\begin{equation}
\nabla_{\mathbf{x}}
\mathbf{m}_t^{\mathbf{o}}(\mathbf{x})
=
\frac{t}{\sigma_t^2}
\boldsymbol{\Sigma}_t^{\mathbf{o}}(\mathbf{x}),
\label{eq:app_conditional_mean_gradient}
\end{equation}
and hence
\begin{equation}
\nabla_{\mathbf{x}}
\mathbf{v}_t^{\mathbf{o}}(\mathbf{x})
=
\frac{\dot{\sigma}_t}{\sigma_t}\mathbf{I}_d
+
\frac{t}{\sigma_t^3}
\boldsymbol{\Sigma}_t^{\mathbf{o}}(\mathbf{x}).
\label{eq:app_velocity_jacobian}
\end{equation}

For $\sigma_{\min}>0$, let
$\sigma_*:=\min\{1,\sigma_{\min}\}>0$. If
\begin{equation}
C_{\Sigma,\mathbf{o}}
:=
\sup_{\substack{
t\in[0,1]\\
\mathbf{x}\in\mathbb{R}^d
}}
\left\|
\boldsymbol{\Sigma}_t^{\mathbf{o}}(\mathbf{x})
\right\|_{\mathrm{op}}
<
\infty,
\label{eq:app_covariance_bound}
\end{equation}
then
\begin{equation}
L_t^{\mathbf{o}}
\le
\frac{
\left|\sigma_{\min}-1\right|
}{\sigma_*}
+
\frac{
tC_{\Sigma,\mathbf{o}}
}{\sigma_*^3},
\label{eq:app_positive_sigma_lipschitz}
\end{equation}
and therefore
\begin{equation}
\Lambda_{\mathbf{o}}
\le
\frac{
\left|\sigma_{\min}-1\right|
}{\sigma_*}
+
\frac{
C_{\Sigma,\mathbf{o}}
}{2\sigma_*^3}
<
\infty.
\label{eq:app_positive_sigma_lambda}
\end{equation}
Thus, together with global well-posedness, bounded posterior covariance
is sufficient for the global spatial regularity used above.

\paragraph{Connectedness of a noise zone.}
Since $\Phi_{\mathbf{o}}$ is a homeomorphism,
$\Phi_{\mathbf{o}}^{-1}$ is continuous. If
$\mathcal{A}_k^{\mathbf{o}}$ is connected, then
\begin{equation}
\mathcal{Z}_k^{\mathbf{o}}
=
\Phi_{\mathbf{o}}^{-1}
\left(
\mathcal{A}_k^{\mathbf{o}}
\right)
\label{eq:app_connected_noise_zone}
\end{equation}
is the continuous image of $\mathcal{A}_k^{\mathbf{o}}$ under
$\Phi_{\mathbf{o}}^{-1}$ and is therefore connected.
Its openness follows from continuity of $\Phi_{\mathbf{o}}$.

\paragraph{Same-mode neighborhood.}
Take $\mathbf{z}\in\mathcal{Z}_k^{\mathbf{o}}$ and
$\mathbf{u}\in\mathbb{R}^d$ satisfying
\begin{equation}
\left\|
\mathbf{u}-\mathbf{z}
\right\|_2
<
R_{\mathbf{o}}(\mathbf{z})
=
e^{-\Lambda_{\mathbf{o}}}
\gamma_{\mathbf{o}}(\mathbf{z}).
\label{eq:app_same_mode_source_radius}
\end{equation}
By ~\eqref{eq:fm_bilipschitz},
\begin{equation}
\left\|
\Phi_{\mathbf{o}}(\mathbf{u})
-
\Phi_{\mathbf{o}}(\mathbf{z})
\right\|_2
<
\gamma_{\mathbf{o}}(\mathbf{z}).
\label{eq:app_same_mode_endpoint_radius}
\end{equation}
Since $\gamma_{\mathbf{o}}(\mathbf{z})$ is the distance from
$\Phi_{\mathbf{o}}(\mathbf{z})$ to the semantic boundary
$\mathcal{B}_A^{\mathbf{o}}$,
$\Phi_{\mathbf{o}}(\mathbf{u})$ remains in
$\mathcal{A}_k^{\mathbf{o}}$. Therefore,
\begin{equation}
B
\left(
\mathbf{z},
R_{\mathbf{o}}(\mathbf{z})
\right)
\subseteq
\mathcal{Z}_k^{\mathbf{o}}.
\label{eq:app_same_mode_neighborhood}
\end{equation}

\paragraph{Cross-mode separation.}
Recall that
\begin{equation}
\mathbb{R}^d
=
\mathcal{B}_A^{\mathbf{o}}
\cup
\bigcup_{k=1}^{K}
\mathcal{A}_k^{\mathbf{o}},
\qquad
\mathcal{B}_A^{\mathbf{o}}
\cap
\mathcal{A}_k^{\mathbf{o}}
=
\varnothing,
\label{eq:app_semantic_partition}
\end{equation}
with the $\mathcal{A}_k^{\mathbf{o}}$ pairwise disjoint. Hence any
continuous curve joining two different semantic regions must intersect
$\mathcal{B}_A^{\mathbf{o}}$.

Let
$\mathbf{z}\in\mathcal{Z}_k^{\mathbf{o}}$ and
$\mathbf{z}'\in\mathcal{Z}_j^{\mathbf{o}}$ with $k\neq j$, and define
\begin{equation}
\mathbf{x}
=
\Phi_{\mathbf{o}}(\mathbf{z}),
\qquad
\mathbf{x}'
=
\Phi_{\mathbf{o}}(\mathbf{z}').
\label{eq:app_cross_mode_endpoints}
\end{equation}
The line segment joining $\mathbf{x}$ and $\mathbf{x}'$ intersects
$\mathcal{B}_A^{\mathbf{o}}$ at some $\mathbf{b}$. Therefore,
\begin{align}
\left\|
\mathbf{x}-\mathbf{x}'
\right\|_2
&=
\left\|
\mathbf{x}-\mathbf{b}
\right\|_2
+
\left\|
\mathbf{x}'-\mathbf{b}
\right\|_2
\nonumber\\
&\ge
\gamma_{\mathbf{o}}(\mathbf{z})
+
\gamma_{\mathbf{o}}(\mathbf{z}').
\label{eq:app_endpoint_separation}
\end{align}
Combining this with ~\eqref{eq:fm_bilipschitz} gives
\begin{align}
\left\|
\mathbf{z}-\mathbf{z}'
\right\|_2
&\ge
e^{-\Lambda_{\mathbf{o}}}
\left[
\gamma_{\mathbf{o}}(\mathbf{z})
+
\gamma_{\mathbf{o}}(\mathbf{z}')
\right]
\nonumber\\
&=
R_{\mathbf{o}}(\mathbf{z})
+
R_{\mathbf{o}}(\mathbf{z}').
\label{eq:app_cross_mode_separation}
\end{align}
Together with the same-mode neighborhood result, this proves
Theorem~\ref{thm:regularity}.

% ============================================================
\subsection{Proof of Proposition~\ref{prop:mode_preservation}
and Corollary~\ref{cor:boundary_region}}
\label{app:practical_error}

\begin{proof}[Proof of Proposition~\ref{prop:mode_preservation}]
Let
$\mathbf{z}\in\mathcal{Z}_k^{\mathbf{o}}$ and define
\begin{equation}
\mathbf{x}
:=
\Phi_{\mathbf{o}}(\mathbf{z}),
\qquad
\widetilde{\mathbf{x}}
:=
\widetilde{\Phi}_{\mathbf{o}}(\mathbf{z}),
\label{eq:app_ideal_practical_endpoints}
\end{equation}
where
$\mathbf{x},\widetilde{\mathbf{x}}\in\mathbb{R}^d$.
If
\begin{equation}
\eta_{\mathbf{o}}(\mathbf{z})
<
\gamma_{\mathbf{o}}(\mathbf{z}),
\label{eq:app_error_smaller_than_margin}
\end{equation}
then
\begin{equation}
\left\|
\widetilde{\mathbf{x}}-\mathbf{x}
\right\|_2
<
\operatorname{dist}
\left(
\mathbf{x},
\mathcal{B}_A^{\mathbf{o}}
\right).
\label{eq:app_endpoint_inside_margin}
\end{equation}
Therefore,
$\widetilde{\mathbf{x}}$
lies strictly inside the largest boundary-free ball centered at
$\mathbf{x}$.
Since
$\mathbf{x}\in\mathcal{A}_k^{\mathbf{o}}$,
it follows that
$\widetilde{\mathbf{x}}\in\mathcal{A}_k^{\mathbf{o}}$.
Hence
\begin{equation}
h_{\mathbf{o}}
\left(
\widetilde{\Phi}_{\mathbf{o}}(\mathbf{z})
\right)
=
h_{\mathbf{o}}
\left(
\Phi_{\mathbf{o}}(\mathbf{z})
\right)
=
k.
\label{eq:app_mode_preservation}
\end{equation}
\end{proof}

To formalize the boundary-sensitive region, let
$\epsilon\in\mathbb{R}_{\ge0}$ be a uniform endpoint-error bound and define
\begin{equation}
\mathcal{S}_{\epsilon}^{\mathbf{o}}
:=
\left\{
\mathbf{z}\in\mathbb{R}^d:
\gamma_{\mathbf{o}}(\mathbf{z})
\le
\epsilon
\right\}.
\label{eq:app_sensitive_region}
\end{equation}
Suppose
\begin{equation}
\eta_{\mathbf{o}}(\mathbf{z})
\le
\epsilon.
\label{eq:app_uniform_endpoint_error}
\end{equation}
For every
$\mathbf{z}\notin\mathcal{S}_{\epsilon}^{\mathbf{o}}$,
we have
\begin{equation}
\eta_{\mathbf{o}}(\mathbf{z})
\le
\epsilon
<
\gamma_{\mathbf{o}}(\mathbf{z}),
\label{eq:app_outside_sensitive_region}
\end{equation}
so Proposition~\ref{prop:mode_preservation} guarantees preservation of the
behavioral mode.

Moreover, for any
$\epsilon_1,\epsilon_2\in\mathbb{R}_{\ge0}$ with
$\epsilon_1\le\epsilon_2$,
\begin{equation}
\mathcal{S}_{\epsilon_1}^{\mathbf{o}}
\subseteq
\mathcal{S}_{\epsilon_2}^{\mathbf{o}}.
\label{eq:app_sensitive_region_monotone}
\end{equation}
Thus, as the endpoint-error bound increases, the set of source points for
which mode preservation is not guaranteed expands monotonically. This
establishes Corollary~\ref{cor:boundary_region}.

% ============================================================
\subsection{Local Guarantee for $\sigma_{\min}=0$}
\label{app:sigma_zero}

When $\sigma_{\min}=0$, we have
$\sigma_t=1-t$, so the global sufficient bound in
~\eqref{eq:app_positive_sigma_lipschitz} is no longer available as
$t\rightarrow1$.
We therefore do not assume a global bi-Lipschitz bound in this case.
Instead, the geometric argument in
Theorem~\ref{thm:regularity} admits a local version on any compact source
region for which the velocity Jacobian remains integrable along the
corresponding trajectories.

Let
$\mathcal{K}\subset\mathbb{R}^d$
be a compact set with nonempty interior.
For each $t\in[0,1]$, define
\begin{equation}
\mathcal{C}_t^{\mathbf{o}}(\mathcal{K})
:=
\operatorname{conv}
\left(
\Phi_{0,t}^{\mathbf{o}}(\mathcal{K})
\right)
\subseteq
\mathbb{R}^d,
\label{eq:app_local_convex_hull}
\end{equation}
where $\operatorname{conv}(\cdot)$ denotes the convex hull.
Define the scalar
\begin{equation}
L_{t,\mathcal{K}}^{\mathbf{o}}
:=
\sup_{
\mathbf{x}\in
\mathcal{C}_t^{\mathbf{o}}(\mathcal{K})
}
\left\|
\nabla_{\mathbf{x}}
\mathbf{v}_t^{\mathbf{o}}(\mathbf{x})
\right\|_{\mathrm{op}}
\in\mathbb{R}_{\ge0}.
\label{eq:app_local_lipschitz}
\end{equation}
Assume that the ODE solution exists uniquely for every
$\mathbf{z}\in\mathcal{K}$ over $t\in[0,1]$ and that
\begin{equation}
\Lambda_{\mathbf{o},\mathcal{K}}
:=
\int_0^1
L_{t,\mathcal{K}}^{\mathbf{o}}
\,\mathrm{d}t
<
\infty.
\label{eq:app_local_lambda}
\end{equation}

For any
$\mathbf{z},\mathbf{z}'\in\mathcal{K}$,
the line segment between
$\Phi_{0,t}^{\mathbf{o}}(\mathbf{z})$
and
$\Phi_{0,t}^{\mathbf{o}}(\mathbf{z}')$
is contained in
$\mathcal{C}_t^{\mathbf{o}}(\mathcal{K})$.
The mean-value inequality therefore gives
\begin{align}
&
\left\|
\mathbf{v}_t^{\mathbf{o}}
\left(
\Phi_{0,t}^{\mathbf{o}}(\mathbf{z})
\right)
-
\mathbf{v}_t^{\mathbf{o}}
\left(
\Phi_{0,t}^{\mathbf{o}}(\mathbf{z}')
\right)
\right\|_2
\nonumber\\
&\qquad\le
L_{t,\mathcal{K}}^{\mathbf{o}}
\left\|
\Phi_{0,t}^{\mathbf{o}}(\mathbf{z})
-
\Phi_{0,t}^{\mathbf{o}}(\mathbf{z}')
\right\|_2.
\label{eq:app_local_velocity_lipschitz}
\end{align}
Applying Gr\"onwall's inequality forward and backward along these
trajectories yields
\begin{equation}
e^{-\Lambda_{\mathbf{o},\mathcal{K}}}
\left\|
\mathbf{z}-\mathbf{z}'
\right\|_2
\le
\left\|
\Phi_{\mathbf{o}}(\mathbf{z})
-
\Phi_{\mathbf{o}}(\mathbf{z}')
\right\|_2
\le
e^{\Lambda_{\mathbf{o},\mathcal{K}}}
\left\|
\mathbf{z}-\mathbf{z}'
\right\|_2.
\label{eq:app_local_bilipschitz}
\end{equation}
Hence $\Phi_{\mathbf{o}}$ is bi-Lipschitz when restricted to
$\mathcal{K}$.

For
$\mathbf{z}\in
\mathcal{Z}_k^{\mathbf{o}}\cap\operatorname{int}(\mathcal{K})$,
define the scalar local radius
\begin{equation}
R_{\mathbf{o},\mathcal{K}}(\mathbf{z})
:=
\min
\left\{
e^{-\Lambda_{\mathbf{o},\mathcal{K}}}
\gamma_{\mathbf{o}}(\mathbf{z}),
\;
\operatorname{dist}
\left(
\mathbf{z},
\mathbb{R}^d\setminus\mathcal{K}
\right)
\right\}
\in\mathbb{R}_{\ge0}.
\label{eq:app_local_radius}
\end{equation}
For every
$\mathbf{u}\in
B(\mathbf{z},R_{\mathbf{o},\mathcal{K}}(\mathbf{z}))$,
we have
$\mathbf{u}\in\mathcal{K}$ and, by
~\eqref{eq:app_local_bilipschitz},
\begin{equation}
\left\|
\Phi_{\mathbf{o}}(\mathbf{u})
-
\Phi_{\mathbf{o}}(\mathbf{z})
\right\|_2
<
\gamma_{\mathbf{o}}(\mathbf{z}).
\label{eq:app_local_endpoint_margin}
\end{equation}
Therefore,
\begin{equation}
B
\left(
\mathbf{z},
R_{\mathbf{o},\mathcal{K}}(\mathbf{z})
\right)
\subseteq
\mathcal{Z}_k^{\mathbf{o}}.
\label{eq:app_local_same_mode}
\end{equation}

Similarly, for
$\mathbf{z}\in\mathcal{Z}_k^{\mathbf{o}}\cap\mathcal{K}$ and
$\mathbf{z}'\in\mathcal{Z}_j^{\mathbf{o}}\cap\mathcal{K}$
with $k\neq j$, the endpoint-margin argument in
~\eqref{eq:app_endpoint_separation}, together with
~\eqref{eq:app_local_bilipschitz}, gives
\begin{equation}
\left\|
\mathbf{z}-\mathbf{z}'
\right\|_2
\ge
e^{-\Lambda_{\mathbf{o},\mathcal{K}}}
\left[
\gamma_{\mathbf{o}}(\mathbf{z})
+
\gamma_{\mathbf{o}}(\mathbf{z}')
\right].
\label{eq:app_local_cross_mode}
\end{equation}
Thus, although a global regularity guarantee is not assumed when
$\sigma_{\min}=0$, the same same-mode neighborhood and cross-mode
separation arguments remain valid locally wherever the FM velocity field
has an integrable spatial Jacobian along the relevant trajectories.

\FloatBarrier
% ============================================================
\section{Additional Details of the Modality Strength Estimator}
\label{app:estimator}

% ------------------------------------------------------------
\subsection{Architecture and Hyperparameters}
\label{app:estimator_hyperparameters}

The estimator is a lightweight two-layer MLP with hidden width 256,
GELU activations, and dropout $0.1$, followed by a scalar sigmoid output.
Its input consists of an eight-step observation history and an eight-step
executed-action history. A LayerNorm operation is applied to the flattened
input.

\begin{wraptable}{r}{0.55\textwidth}
    \centering
    \caption{\textbf{Hyperparameters of the Modality Strength Estimator.}}
    \label{tab:estimator_hyperparameters}
    \small
    \setlength{\tabcolsep}{5pt}
    \begin{tabular}{lc}
        \toprule
        Hyperparameter & Value \\
        \midrule
        Observation/action history length & 8 / 8 \\
        Future-action horizon & 8 \\
        KNN projection dimension & 128 \\
        Candidate/retained neighbors & 40 / 20 \\
        MLP hidden dimensions & $256,256$ \\
        Dropout & 0.1 \\
        Optimizer and learning rate & AdamW, $3{\times}10^{-4}$ \\
        Estimator batch size & 128 \\
        Calibration threshold $\epsilon_w$ & $10^{-3}$ \\
        Active-percentile saturation $\eta$ & 0.95 \\
        Persistence floor $q_{\mathrm{floor}}$ & 0.05 \\
        Persistence factor $\lambda_q$ & 0.75 \\
        \bottomrule
    \end{tabular}
    \vspace{-40pt}
    
\end{wraptable}

The estimator is optimized using a confidence-weighted Huber loss,
\begin{equation}
    \mathcal{L}_{\mathrm{est}}
    =
    \frac{
        \sum_i c_i\,
        \ell_{\mathrm{Huber}}(w_i-r_i)
    }{
        \sum_i c_i+\epsilon
    },
    \label{eq:app_estimator_loss}
\end{equation}
where $c_i$ measures the reliability of the local neighborhood.

The principal hyperparameters shared across tasks are summarized in
Table~\ref{tab:estimator_hyperparameters}.

% ------------------------------------------------------------
\subsection{Neighborhood Construction and Supervision}
\label{app:estimator_neighbors}

Each observation is represented by a 512-dimensional frozen ResNet-18 feature concatenated with the task-specific normalized proprioceptive state. The flattened observation history is projected to 128 dimensions.

For each query, we first retrieve $K_{\mathrm{cand}}=40$ candidates using
the projected observation histories and then re-rank them using both
observation and executed-action distances:
\begin{equation}
    d_{ij}
    =
    \frac{d_{ij}^{\mathrm{obs}}}
         {\operatorname{median}(d_{i\cdot}^{\mathrm{obs}})}
    +
    \lambda_h
    \frac{d_{ij}^{\mathrm{act}}}
         {\operatorname{median}(d_{i\cdot}^{\mathrm{act}})},
    \qquad
    \lambda_h=1.
    \label{eq:app_knn}
\end{equation}
The $K_{\mathrm{NN}}=20$ closest candidates are retained. The query itself
and samples from the same episode within eight time steps are excluded to
avoid trivial temporal neighbors. Each retained neighbor is assigned
\begin{equation}
    \alpha_{ij}
    =
    \frac{\exp(-d_{ij}/\tau)}
         {\sum_{\ell\in\mathcal{N}_i}\exp(-d_{i\ell}/\tau)},
    \qquad
    \tau=1.
    \label{eq:app_neighbor_weight}
\end{equation}

Let $\mathbf{y}_j$ denote the normalized future-action chunk of neighbor
$j$. For each mode $k$, we compute its local probability and weighted
future-action mean as
\begin{equation}
    p_{ik}
    =
    \sum_{j\in\mathcal{N}_{ik}}\alpha_{ij},
    \qquad
    \bar{\mathbf{y}}_{ik}
    =
    \frac{1}{p_{ik}}
    \sum_{j\in\mathcal{N}_{ik}}
    \alpha_{ij}\mathbf{y}_j.
    \label{eq:app_local_statistics}
\end{equation}
With
$\bar{\mathbf{y}}_i=\sum_k p_{ik}\bar{\mathbf{y}}_{ik}$,
the between-mode and within-mode variations are
implemented as
\begin{align}
    \mathrm{Var}_{B,i}
    &=
    \sum_k
    p_{ik}
    \left\|
        \bar{\mathbf{y}}_{ik}-\bar{\mathbf{y}}_i
    \right\|_{\mathrm{ms}}^2,
    \\
    \mathrm{Var}_{W,i}
    &=
    \sum_k
    \sum_{j\in\mathcal{N}_{ik}}
    \alpha_{ij}
    \left\|
        \mathbf{y}_j-\bar{\mathbf{y}}_{ik}
    \right\|_{\mathrm{ms}}^2,
    \label{eq:app_variations}
\end{align}

The numerical stabilizer $\delta$ in the main-text supervision target is
set to the 10th percentile of positive within-mode variations on the
training split. The confidence weight $c_i$ combines neighborhood quality,
measured from the retrieval distance, with effective mode support, thereby
suppressing unreliable targets from distant or poorly represented
neighborhoods.

% ------------------------------------------------------------
\subsection{Calibration and Temporal Persistence}
\label{app:estimator_calibration}

Because the numerical range of the estimator output $w$ varies across
tasks, we calibrate it using the empirical distribution of predictions on
the training split. Let
\begin{equation}
    \widehat{F}_{+}(w)
    =
    \frac{
        \sum_{i\in\mathcal{D}_{\mathrm{train}}}
        \mathbbm{1}[\epsilon_w<w_i\leq w]
    }{
        \sum_{i\in\mathcal{D}_{\mathrm{train}}}
        \mathbbm{1}[w_i>\epsilon_w]
    }
    \label{eq:app_active_cdf}
\end{equation}
denote the empirical CDF of active predictions. We use
\begin{equation}
    q(w)
    =
    \begin{cases}
        0,
        & w\leq\epsilon_w,\\[1mm]
        \displaystyle
        \min\left(
            \frac{\widehat{F}_{+}(w)}{\eta},
            1
        \right),
        & w>\epsilon_w,
    \end{cases}
    \label{eq:app_percentile_calibration}
\end{equation}
where $\epsilon_w=10^{-3}$ and $\eta=0.95$. Thus, inactive predictions
remain zero, while the 95th percentile of active training predictions maps
to $q=1$. The calibration is computed solely from the training split and
fixed thereafter.

During rollout, temporal persistence is applied as
\begin{equation}
    q_t^{\mathrm{eff}}
    =
    \max\left(
        q_t,\,
        q_{\mathrm{floor}},\,
        \lambda_q q_{t-1}^{\mathrm{eff}}
    \right),
    \label{eq:app_temporal_persistence}
\end{equation}
where $q_{\mathrm{floor}}=0.05$ and $\lambda_q=0.75$. This prevents a
recently detected multimodal state from immediately losing source
separation at subsequent policy queries. For independently sampled
training windows, we use
$q^{\mathrm{eff}}=\max(q,q_{\mathrm{floor}})$.
For simplicity, we denote $q^{\mathrm{eff}}$ as $q$ when there is no
ambiguity.

% ============================================================
\section{Additional Details of Mode-Oriented Source Shaping}
\label{app:source_shaping}

\subsection{Branch-Critical Region Selection.}
\label{sec:branch_critical_region}

The multimodality supervision $r_i$ defined in
Section~\ref{sec:modality_strength_estimator} is also used offline to identify
the training region where future behaviors are most clearly separated.
Specifically, we rank training windows by $r_i$ and select the region with the
highest multimodality. For each mode $k$, $\bar{\mathbf{a}}_k$ is computed as
the empirical mean future-action chunk within this region. We then center these
means according to ~\eqref{eq:action_mean} and construct $\mathbf{U}$ as an orthonormal basis of the
resulting mode-difference vectors. All statistics are computed solely from the
training demonstrations and are fixed before FM training and deployment.

% ------------------------------------------------------------
\subsection{Adaptive Separation Schedule}
\label{app:adaptive_separation}

We first normalize the active range of the modality strength and apply a
monotone gain
\begin{equation}
    g(q)
    =
    \left[
        \operatorname{clip}
        \left(
            \frac{q-q_{\mathrm{floor}}}
                 {1-q_{\mathrm{floor}}},
            0,
            1
        \right)
    \right]^{\gamma},
    \qquad
    \gamma>0.
    \label{eq:app_gain_schedule}
\end{equation}

The source-center scale in ~\eqref{eq:moss_source} is parameterized as
\begin{equation}
    \alpha(q)
    =
    \alpha_{\min}
    +
    \left(
        \alpha_{\max}-\alpha_{\min}
    \right)g(q),
    \label{eq:app_alpha_schedule}
\end{equation}
and the target separation ratio as
\begin{equation}
    \rho(q)
    =
    \rho_{\min}
    +
    \left(
        \rho_{\max}-\rho_{\min}
    \right)g(q).
    \label{eq:app_rho_schedule}
\end{equation}

Thus, increasing $q$ simultaneously moves the mode-specific source centers
farther apart and increases their target separation relative to their
spread. The corresponding $\sigma_{\parallel}(q)$ is then obtained from
~\eqref{eq:moss_sigma}. The clipping bounds prevent the
source distribution from becoming excessively concentrated or diffuse,
while $\sigma_{\perp}$ remains fixed.

We use the same source-shaping
hyperparameters across all tasks, as summarized in
Table~\ref{tab:source_shaping_hyperparams}.

\begin{table}[h]
    \centering
    \captionsetup{justification=centering}
    \caption{\textbf{Source-shaping hyperparameters used in MOSS.}}
    \label{tab:source_shaping_hyperparams}
    \small
    \setlength{\tabcolsep}{6pt}
    \begin{tabular}{cccccccc}
        \toprule
        $\alpha_{\min}$ &
        $\alpha_{\max}$ &
        $\rho_{\min}$ &
        $\rho_{\max}$ &
        $\sigma_{\perp}$ &
        $\sigma_{\parallel,\min}$ &
        $\sigma_{\parallel,\max}$ &
        $\gamma$ \\
        \midrule
        0.0 & 1.0 & 1.25 & 2.5 & 1.0 & 0.25 & 1.0 & 0.5 \\
        \bottomrule
    \end{tabular}
    
\end{table}

% ------------------------------------------------------------
\subsection{Training and Inference Procedure}
\label{app:moss_training}

Algorithm~\ref{alg:moss} summarizes the complete procedure. The modality strength
estimator and its calibration are fitted before FM training and remain
frozen thereafter. During training, the source distribution is selected
using the demonstration mode label; during deployment, it is selected
using the requested mode.

\section{Detailed Experimental Setup}
\label{app:experiment_setup}

\subsection{Tasks and Mode Definitions}
\label{app:tasks}

We evaluate MOSS on $6$ simulated tasks and $3$ real-world tasks. The $4$ bimodal simulation tasks are \emph{Push Cube},
\emph{Grasp Eyeglasses}, \emph{Stack Cube}, and \emph{Put Wine}.
The $2$ multi-mode tasks are \emph{2D Navigation} and \emph{Place Cube}.
The real-world experiments comprise \emph{Place Cups},
\emph{Pick Veg}, and \emph{Open Drawer}. Mode annotations
identify distinct successful strategies rather than task success itself.

The mode semantics used in the three repeated simulation comparisons and
the real-world experiments are summarized in
Table~\ref{tab:app_mode_semantics}.

\begin{table*}[t]
    \centering

    \captionsetup{justification=centering}
    \caption{\textbf{Semantic definitions of the two behavioral modes used in the
    controlled evaluations.}}
    \label{tab:app_mode_semantics}
    \small
    \begin{tabular}{lll}
        \toprule
        Task & Mode 1 & Mode 2 \\
        \midrule
        Push Cube
        & Push around the left side
        & Push around the right side \\
        Grasp Eyeglasses
        & Positive-$y$ approach
        & Negative-$y$ approach \\
        Stack Cube
        & Cube A on Cube B
        & Cube B on Cube A \\
        \midrule
        Place Cups
        & Grasp the left cup
        & Grasp the right cup \\
        Pick Veg
        & Pick the left radish first
        & Pick the right corn first \\
        Open Drawer
        & Open the lower blue drawer
        & Open the upper green drawer \\
        \bottomrule
    \end{tabular}
    
\end{table*}

\subsection{Baselines}

\paragraph{GMM-Sourced FM.}
We implement a Gaussian-mixture source baseline following the general multimodal source-space formulations explored in prior work~\citep{luo2026mmfm,zhang2026slfm}. The source distribution is modeled as a mixture of Gaussian components whose parameters are estimated from the training data, and samples from different components are transported using a shared flow-matching model. We adapt the source-space construction to our policy architecture and experimental setting while keeping the remaining training configuration consistent with FM-DiT.
For components without an available official implementation, we follow the formulation described in the corresponding paper and adapt it to our policy architecture.

\paragraph{Mode-conditioned FM.}
Our goal is to compare source-space mode control against direct mode conditioning under the same FM policy backbone. PF-DAG (Lei et al., 2026) contains a mode-conditioned action-generation stage, but its complete pipeline additionally predicts the primary mode, addressing the distinct problem of automatically inferring multimodal structure. Since our setting assumes that the desired semantic mode is explicitly specified at deployment, we isolate the mode-conditioned generation mechanism and directly provide the requested mode label to the FM policy. Specifically, the mode label is embedded and injected as an additional condition into the velocity network throughout transport. MOSS and Mode-conditioned FM use the same visual encoder, velocity-field architecture, action horizon, optimizer, training data, and evaluation protocol. The comparison is therefore intended to isolate where the given mode information enters the generative process—through the velocity field or through the source distribution—rather than to compare against the full PF-DAG system.

\begin{algorithm}[t]
\caption{MOSS Training and Mode Intervention}
\label{alg:moss}
\begin{algorithmic}[1]
\Require Labeled demonstrations
$\mathcal{D}=\{(\mathbf{o}_{\le t_i},\mathbf{a}_{<t_i},
\mathbf{y}_i,k_i)\}$,
$k_i\in\{1,\ldots,K\}$
\Ensure Frozen estimator $s_\phi$, calibration map $\mathcal{C}$,
and shared velocity field $\mathbf{v}_\theta$

\Statex \textbf{Stage I: Training Modality Strength Estimator}
\For{each query $i\in\mathcal{D}$}
    \State Retrieve local neighbors using
    $(\mathbf{o}_{\le t_i},\mathbf{a}_{<t_i})$
    \State Compute modality target $r_i$
\EndFor
\State Train
$s_\phi(\mathbf{o}_{\le t_i},\mathbf{a}_{<t_i})\approx r_i$
\State Fit calibration map $\mathcal{C}$ on the training split
\State Freeze $s_\phi$ and $\mathcal{C}$

\Statex \textbf{Stage II: Training FM Policy}
% \State Compute
% $\{\bar{\mathbf{a}}_k\}_{k=1}^{K}$,
% $\{\tilde{\mathbf{a}}_k\}_{k=1}^{K}$,
% and $\mathbf{U}$
\State Identify the branch-critical region from $\{r_i\}$
\State Compute
$\{\bar{\mathbf{a}}_k\}_{k=1}^{K}$,
$\{\tilde{\mathbf{a}}_k\}_{k=1}^{K}$,
and $\mathbf{U}$

\For{each minibatch
$(\mathbf{o}_{\le t},\mathbf{a}_{<t},\mathbf{y},k)$}
    \State
    $q\gets
    \mathcal{C}\!\left(
        s_\phi(\mathbf{o}_{\le t},\mathbf{a}_{<t})
    \right)$
    \State
    $q^{\mathrm{eff}}\gets\max(q,q_{\mathrm{floor}})$
    \State Sample
    $\mathbf{x}_0\sim
    p(\mathbf{x}_0\mid k,q^{\mathrm{eff}})$
    \State Sample $\tau\sim\mathcal{U}(0,1)$ and set
    $\mathbf{x}_\tau=(1-\tau)\mathbf{x}_0+\tau\mathbf{y}$
    \State Update only $\theta$ using the standard FM objective
\EndFor

\Statex \textbf{Stage III: Deployment}
\State Initialize $q_{-1}^{\mathrm{eff}}\gets0$
\For{each policy query $t$}
    \State
    $q_t\gets
    \mathcal{C}\!\left(
        s_\phi(\mathbf{o}_{\le t},\mathbf{a}_{<t})
    \right)$
    \State
    $q_t^{\mathrm{eff}}
    \gets
    \max\left(
        q_t,\,
        q_{\mathrm{floor}},\,
        \lambda_q q_{t-1}^{\mathrm{eff}}
    \right)$
    \State Select the desired mode $k_t$
    \State Sample
    $\mathbf{x}_0\sim
    p(\mathbf{x}_0\mid k_t,q_t^{\mathrm{eff}})$
    \State Integrate
    $\mathbf{v}_\theta(\cdot,\cdot,\mathbf{o}_{\le t})$
    from flow time $0$ to $1$ to obtain
    $\hat{\mathbf{y}}_t$
    \State Execute the current action slice and update
    $\mathbf{a}_{<t+1}$
\EndFor
\end{algorithmic}
\end{algorithm}

\paragraph{Unconditional Baselines.}
FM-DiT~\citep{lipman2023flow,peebles2023scalable},
DDPM-DiT~\citep{ho2020denoising,peebles2023scalable},
ACT~\citep{zhao2023learning},
BET~\citep{shafiullah2022behavior},
A2A~\citep{jia2026a2a}, and
MARS Policy~\citep{jia2026marspolicy}
are reproduced following their original formulations and implementation
settings.

\subsection{Simulation Setup}
\label{app:training_details}

\begin{wraptable}{r}{0.55\textwidth}
    \centering
    \vspace{-20pt}
    \captionsetup{justification=centering}
    \caption{\textbf{Shared configuration of the FM-based policies.}}
    \label{tab:policy_hyperparameters}
    \small
    \begin{tabular}{lc}
        \toprule
        Hyperparameter & Value \\
        \midrule
        Observation history & $8$ steps \\
        Predicted action horizon & $16$ steps \\
        Actions executed per query & $8$ steps \\
        Dropout & $0.1$ \\
        FM integration steps & $10$ \\
        Optimizer & AdamW \\
        Learning rate & $10^{-4}$ \\
        Weight decay & $10^{-6}$ \\
        Batch size & $32$ \\
        Maximum train / validation batches per epoch & $250$ / $250$ \\
        \bottomrule
        \vspace{-40pt}
    \end{tabular}
    
\end{wraptable}

For \emph{Push Cube}, \emph{Grasp Eyeglasses}, and
\emph{Stack Cube}, FM-DiT, Mode-conditioned FM, and MOSS use the same
visual encoder, velocity-field architecture, prediction horizon,
optimizer, and FM solver configuration. Their verified common settings are
listed in Table~\ref{tab:policy_hyperparameters}. Training duration and
checkpoint-selection schedules are kept identical across the compared methods for each task.

\subsection{Real-World Robot Setup}

The real-world experiments use a Galaxea R1 Lite robot. Each policy
receives observations from a head camera and a right-wrist camera together
with the seven-dimensional right-arm state, comprising six joint positions
and the gripper position. Both RGB observations are resized to
$256\times256$. The corresponding action is also seven-dimensional and
specifies the absolute positions of the six right-arm joints and the
gripper.

All three real-world policies use an observation history of $8$ steps and
predict action chunks with a horizon of $16$. Up to $8$ actions from each
chunk are executed before the next policy query. The robot is controlled
in closed loop at $15\,\mathrm{Hz}$.

For \emph{Place Cups}, the demonstrations contain $52$ episodes per
mode after filtering. For \emph{Pick Veg}, the dataset contains $55$
episodes per mode. For \emph{Open Drawer}, the converted dataset
contains $51$ valid episodes per mode after discarding one anomalous
episode. The mode labels are assigned per episode according to the
behavioral choice described in Table~\ref{tab:app_mode_semantics}.

\subsection{Evaluation Protocol and Metrics}
\label{app:metrics}

For simulation, each explicitly requested mode is evaluated with $50$
rollouts. For methods without a mode-selection interface, only task
success and the empirical behavior distribution are reported.

For each real-world mode-controllable policy, we conduct $20$ trials for
each requested mode. The corresponding $\operatorname{GR}_k$ is computed
over the $20$ trials requesting mode $k$, while the reported SR is averaged
over the two requested-mode conditions. FM-DiT is evaluated with $20$
unconditional trials per task. Real-world results are rounded to two
decimal places.

Let $S_i\in\{0,1\}$ indicate whether rollout $i$ completes the task, and
let $M_i$ denote its executed semantic mode. The success rate is
$\operatorname{SR}=\frac{1}{N}\sum_{i=1}^{N} S_i$.
For a requested mode $k$, the guidance rate is
    $\operatorname{GR}_{k}
    =
    \frac{1}{N_k}
    \sum_{i=1}^{N_k}
    \mathbb{I}\!\left[S_i=1 \land M_i=k\right]$.
A rollout contributes to GR only when it both succeeds and realizes the
requested semantic behavior. Consequently,
$\operatorname{GR}_{k}\leq\operatorname{SR}_{k}$. 

% ============================================================
\section{Classifier-Free Guidance for Mode-Conditioned FM}
\label{app:cfg_results}

Mode-conditioned FM uses classifier-free guidance
\begin{equation}
    \mathbf{v}_{\mathrm{CFG}}
    =
    \mathbf{v}_{\mathrm{uncond}}
    +
    s_{\mathrm{CFG}}
    \left(
        \mathbf{v}_{\mathrm{cond}}
        -
        \mathbf{v}_{\mathrm{uncond}}
    \right),
    \label{eq:app_cfg}
\end{equation}
where $s_{\mathrm{CFG}}$ is the inference-time guidance scale. The primary
real-world evaluation uses $s_{\mathrm{CFG}}=1$. We additionally increase
the scale for the requested mode on which the default setting exhibits
mode collapse. Table~\ref{tab:cfg_sweep} reports these targeted
comparisons.

\begin{wraptable}{r}{0.55\textwidth}
    \centering
    \vspace{-12pt}
    \captionsetup{justification=centering}
    \caption{\textbf{CFG comparisons for modes on which
    Mode-conditioned FM collapses under the default guidance scale.}}
    \vspace{-5pt}
    \label{tab:cfg_sweep}
    \small
    \begin{tabular}{llccc}
        \toprule
        Task & Requested mode
        & $s_{\mathrm{CFG}}$ & SR & GR \\
        \midrule
        Place Cups
        & Mode 1 & $1$ & $0.55$ & $0.15$ \\
        & Mode 1 & $3$ & $0.70$ & $0.65$ \\
        \midrule
        Pick Veg
        & Mode 1 & $1$ & $0.60$ & $0.00$ \\
        & Mode 1 & $3$ & $0.45$ & $0.20$ \\
        \midrule
        Open Drawer
        & Mode 2 & $1$ & $0.85$ & $0.00$ \\
        & Mode 2 & $3$ & $0.75$ & $0.25$ \\
        \bottomrule
    \end{tabular}
    
\end{wraptable}

Increasing CFG improves mode-1 guidance on \emph{Place Cups} from
$0.15$ to $0.65$. However, the effect does not generalize consistently.
On \emph{Pick Veg}, increasing the scale raises GR only to
$0.20$ while reducing SR from $0.60$ to $0.45$. Similarly, on
\emph{Open Drawer}, increasing the scale raises mode-2 GR
from $0.00$ to $0.25$ but reduces SR from $0.85$ to $0.75$. These results
indicate that inference-time amplification of the mode condition cannot
reliably resolve the asymmetric mode collapse observed in
Mode-conditioned FM.

% ============================================================
% Imbalanced Training
% ============================================================
\section{Imbalanced Training and Data Efficiency}
\label{app:imbalance_results}

For each task, mode 1 is treated as the minority mode. MOSS and
Mode-conditioned FM are trained on $10$ mode-1 and $90$ mode-2
demonstrations, whereas unconditional FM-DiT is trained only on the
same $10$ mode-1 demonstrations. Evaluation always requests mode 1.

Each method is trained independently. At each checkpoint, we evaluate the requested mode and
report its SR and GR. The solid curves in
Figure~\ref{fig:imbalanced_efficiency} show the mean GR across the $3$
seeds, while the shaded regions indicate $\pm 1$ standard deviation
computed from the seed-wise results.

The horizontal axis reports the cumulative number of sampled minority-mode
training windows rather than the epoch index. This normalization allows
the methods to be compared according to their exposure to target-mode
data, while MOSS can additionally benefit from the majority-mode demonstrations.

\FloatBarrier
% ============================================================
\section{Additional Ablation Studies}
\label{app:ablation}

\subsection{Ablation Variants}
\label{app:ablation_source}

The ablations in Table~\ref{tab:ablation} compare $4$ source-distribution variants to the
source distribution:

\begin{itemize}
    \item \textbf{Single Gaussian:} a standard Gaussian is partitioned by
    the sign of its projection onto the mode axis at a selected action
    position, without displaced Gaussian centers.
    \item \textbf{Single-step separation:} mode-specific Gaussian
    displacement is applied only at the branching point.
    \item \textbf{Without adaptation:} fixed mode-specific source
    separation is applied across the complete action chunk.
    \item \textbf{MOSS:} full-chunk separation is adjusted using the
    estimated modality strength.
\end{itemize}

The single-position variants leave most source dimensions shared, whereas
the full-chunk variants maintain a consistent mode identity across the
predicted action sequence.

\subsection{Fixed versus Adaptive Training Curves}
\label{app:ablation_adaptive}

To compare convergence independently of the final checkpoints reported in
Table~\ref{tab:ablation}, we evaluate fixed full-chunk shaping and adaptive
MOSS every $10$ epochs from epoch $10$ to epoch $100$, as shown in Figure~\ref{fig:app_ablation_efficiency}. Each point contains
$50$ rollouts for each requested mode.

\begin{figure*}[t]
    \centering
    \includegraphics[width=\textwidth]{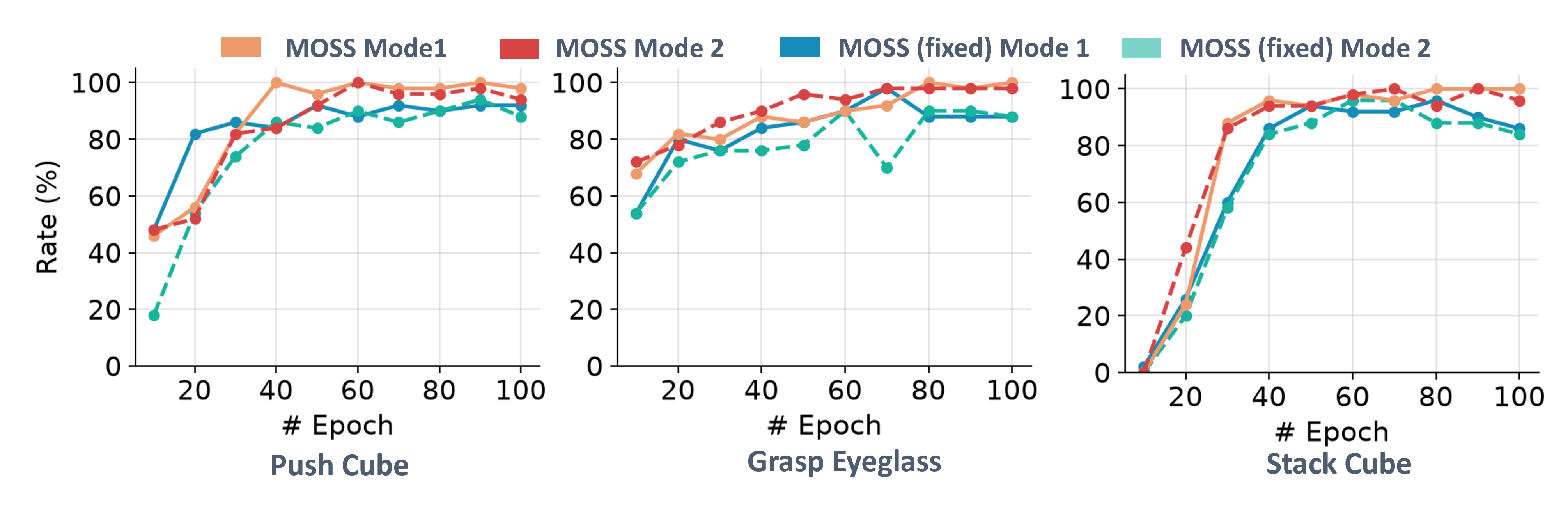}
    \vspace{-20pt}
    \captionsetup{justification=centering}
    \caption{\textbf{Checkpoint-wise GR of fixed
    full-chunk source shaping and adaptive MOSS.}}
    \label{fig:app_ablation_efficiency}
\end{figure*}

% \begin{table}[t]
\begin{wraptable}{r}{0.55\textwidth}
    \vspace{-20pt}
    \centering
    \caption{\textbf{Offline source-alignment and squared-transport statistics.}
    Lower values indicate better alignment or a shorter source-to-target
    transport.}
    \label{tab:app_source_alignment}
    \vspace{-5pt}
    \small
    \setlength{\tabcolsep}{4pt}
    \begin{tabular}{llccc}
        \toprule
        Task & Source
        & $D_{\mu}$ & $D_{\mathrm{tr}}$ & $D_{\parallel}$ \\
        \midrule
        Push Cube
        & Fixed    & $2.997$ & $174.582$ & $2.767$ \\
        & Adaptive & $\mathbf{2.884}$ & $\mathbf{172.823}$ & $\mathbf{0.931}$ \\
        \midrule
        Grasp Eyeglasses
        & Fixed    & $3.032$ & $201.079$ & $4.649$ \\
        & Adaptive & $\mathbf{2.945}$ & $\mathbf{200.983}$ & $\mathbf{4.565}$ \\
        \midrule
        Stack Cube
        & Fixed    & $3.346$ & $209.991$ & $5.884$ \\
        & Adaptive & $\mathbf{3.298}$ & $\mathbf{208.754}$ & $\mathbf{4.683}$ \\
        \bottomrule
    \end{tabular}
    \vspace{-10pt}
% \end{table}
\end{wraptable}

Across the three tasks, adaptive MOSS reaches consistently high guidance
rates for both requested modes and remains stable at later checkpoints.
At epoch $100$, its mean GR across the two modes reaches $0.96$, $0.99$,
and $0.98$ on \emph{Push Cube}, \emph{Grasp Eyeglasses}, and \emph{Stack Cube}, compared with
$0.90$, $0.88$, and $0.85$ for fixed shaping, respectively.
The checkpoint-wise results show that the improvement of adaptive source
shaping persists throughout training rather than arising from a particular
final checkpoint.

\subsection{Source Alignment and Transport Cost}
\label{app:transport_analysis}

We conduct an offline analysis in normalized action space using every
available training window. 

Let $\mathbf{u}$ be the normalized full-chunk mode axis,
$\bar{\mathbf{a}}_k$ the empirical normalized action-chunk mean of mode
$k$, and $\boldsymbol{\mu}_k(q_i)$ the source mean used for window $i$.
We measure center-to-target alignment by
\begin{equation}
    D_{\mu}
    =
    \frac{1}{N}
    \sum_{i=1}^{N}
    \left\|
        \boldsymbol{\mu}_{k_i}(q_i)
        -
        \bar{\mathbf{a}}_{k_i}
    \right\|_2 .
    \label{eq:app_d_mu}
\end{equation}

For a normalized target action chunk $\mathbf{a}_i$ and source sample
$\mathbf{z}_{i,j}$, the transport statistic is the squared displacement
\begin{equation}
    D_{\mathrm{tr}}
    =
    \frac{1}{NM}
    \sum_{i=1}^{N}\sum_{j=1}^{M}
    \left\|
        \mathbf{z}_{i,j}-\mathbf{a}_i
    \right\|_2^2 .
    \label{eq:app_transport_distance}
\end{equation}
Its component along the mode-discriminative direction is
\begin{equation}
    D_{\parallel}
    =
    \frac{1}{NM}
    \sum_{i=1}^{N}\sum_{j=1}^{M}
    \left(
        \left\langle
        \operatorname{vec}(\mathbf{z}_{i,j}-\mathbf{a}_i),
        \operatorname{vec}(\mathbf{u})
        \right\rangle
    \right)^2 ,
    \label{eq:app_transport_parallel}
\end{equation}
and $D_{\perp}=D_{\mathrm{tr}}-D_{\parallel}$.

Adaptive shaping reduces $D_{\mu}$ and total squared transport in all three
tasks. The difference is most visible in $D_{\parallel}$, particularly on
Push Cube and Stack Cube. This is expected because MOSS modifies the
source primarily along the mode-discriminative direction, while the much
larger orthogonal transport component is largely shared by both variants.
\end{document}